\documentclass[journal]{IEEEtran}
\usepackage{algorithm}
\usepackage{multirow}
\usepackage[table ]{xcolor}
\usepackage{array}
\usepackage{algpseudocode}
\usepackage{amsmath}
\usepackage{cite}
\usepackage{booktabs}
\usepackage{graphicx}
\usepackage{float}
\usepackage{threeparttable}
\usepackage{epstopdf}
\usepackage{xfrac}
\usepackage{makecell}
\usepackage{amssymb}
\usepackage{subfig}
\usepackage{mathrsfs}

\usepackage{pifont}
\usepackage{amssymb}
\usepackage{longtable}
\usepackage{subcaption}
\usepackage{amsthm}
\usepackage{bm} 
\usepackage{pdfpages}

\newtheorem{definition}{Definition}

\newtheorem{axiom}{Axiom}
\newtheorem{theorem}{Theorem}

\newtheorem{lemma}{Lemma}
\newtheorem{remark}{Remark}

\newcolumntype{L}[1]{>{\raggedright\arraybackslash}p{#1}}

\begin{document}

\title{Fairness-Aware Performance Evaluation for Multi-Party Multi-Objective Optimization}

\author{Zifan~Zhao, Peilan~Xu,~\IEEEmembership{Member,~IEEE}, Wenjian~Luo,~\IEEEmembership{Senior Member,~IEEE}
 % <-this % stops a space
    \thanks{Accepted for publication in IEEE Transactions on Evolutionary Computation. The final published version will be available on IEEE Xplore.}
	\thanks{This work is partly supported by the National Natural Science Foundation of China (Grant No. 62602284, U23B2058, 62576121), Shenzhen Fundamental Research Program (Grant No. JCYJ20220818102414030), Shenzhen Science and Technology Program (Grant No. ZDSYS20210623091809029). (\textit{Corresponding author: Peilan Xu.}) }
	\thanks{Zifan Zhao is with Waterford Institute, Nanjing University of Information Science and Technology, Nanjing 210044, China.
		
		Peilan Xu is with School of Artificial Intelligence, Nanjing University of Information Science and Technology, Nanjing 210044, China.

    Wenjian Luo is with Guangdong Provincial Key Laboratory of Novel Security Intelligence Technologies, Institute of Cyberspace Security, School of Computer Science and Technology, Harbin Institute of Technology, Shenzhen 518055, China.}
	
	\thanks{Email: 202383920026@nuist.edu.cn, xpl@nuist.edu.cn, luowenjian@hit.edu.cn.}
}

\maketitle

\begin{abstract}
	In multi-party multi-objective optimization problems (MPMOPs), solution sets are commonly evaluated using classical performance indicators, such as inverted generational distance (IGD) and hypervolume (HV), aggregated across decision makers (DMs). However, such mean-based evaluations may implicitly favor certain parties by assuming comparable geometric approximation quality across heterogeneous party-wise Pareto fronts (PFs). Moreover, prevailing notions of MPMOP optimality are largely restricted to strictly common Pareto-optimal solutions, which represent a narrow and sometimes impractical form of cooperation. These limitations obscure whether a solution set reflects balanced relative gains or meaningful consensus among heterogeneous DMs. To address these issues, this paper develops a fairness-aware performance evaluation framework grounded in a generalized notion of consensus solutions. From a cooperative game-theoretic perspective, we formalize four axioms for fairness-aware MPMOP evaluation. By introducing a concession rate vector to quantify acceptable compromises, we generalize the classical definition of MPMOP optimal solutions and embed classical performance metrics into a Nash-product-based evaluation framework, which is theoretically shown to satisfy all four proposed axioms. To support empirical validation, we construct benchmark problems with consensus-deficient negotiation structures. The empirical study covers MPMOP instances with up to three DMs and up to five objectives per DM. Experimental results show that, compared with classical indicators such as meanIGD and meanHV, the proposed framework distinguishes algorithmic performance consistently with consensus-aware fairness considerations. Algorithms converging toward strictly common solutions receive higher scores when such solutions exist, whereas algorithms effectively covering commonly acceptable regions are favored when strict common solutions are absent.

\begin{IEEEkeywords}
Multi-party Multi-objective Optimization, Fairness-aware Evaluation, Nash Bargaining, Performance metrics.
\end{IEEEkeywords}
\end{abstract}

\IEEEpeerreviewmaketitle

\section{Introduction}
\label{sec: introduction}
%MPMOP和多目标博弈的背景
Multi-party multi-objective optimization problems (MPMOPs) arise in decision-making scenarios where multiple decision makers (DMs) simultaneously pursue their own sets of objectives, which may conflict both within and across parties~\cite{liu2020evolutionary}. Typical applications include UAV path planning, where regulatory authorities emphasize safety and compliance, while commercial operators prioritize cost efficiency and operational flexibility~\cite{chen2025novel},\cite{chen2024evolutionary}. In such scenarios, no centralized authority exists to aggregate heterogeneous objectives into a single decision criterion, and feasible outcomes must emerge from compromises among self-interested and heterogeneous DMs. 

% 多方多目标进化社区
Within evolutionary computation, population-based metaheuristics have demonstrated strong effectiveness in approximating PFs, owing to their ability to maintain diversity and to explore trade-offs simultaneously~\cite{HuangLi2024_Survey, xue2025multiple, Wang2024_Surrogate}. Extensions to multi-party scenarios exploit this population-based nature to incorporate the perspectives of multiple DMs. Representative approaches include multi-party non-dominated sorting~\cite{she2021new,she2023multiparty}, which extends NSGA-II~\cite{deb2002fast} by intersecting party-specific Pareto layers to identify mutually acceptable solutions, multi-objective immune algorithms~\cite{chen2024evolutionary} incorporating multi-party dominance relations, and party-by-party evolutionary update strategies~\cite{chang2022multiparty,chang2023biparty} based on MOEA/D~\cite{zhang2007moea}. Indicator-based algorithms further integrate classical indicator-driven selection with multi-party dominance relations to guide the search toward compromise-oriented solutions~\cite{song2024indicator}. In addition, optimization methods tailored for discrete multi-party problems~\cite{zhang2025multi}, theoretical analyses of runtime and convergence behavior~\cite{sun2025runtime}, privacy-preserving mechanisms~\cite{she2023privacy}, and benchmark problems specifically designed for MPMOPs~\cite{song2022multiobjective,luo2024benchmark} have also been actively investigated.

%现有指标的问题
%公平性问题
Despite these advances, the evaluation of solution sets for MPMOPs still largely relies on classical performance metrics such as IGD~\cite{zhang2008multiobjective} and HV~\cite{zhang2007moea}, typically aggregated across parties to form metrics such as meanIGD and meanHV. Such aggregation implicitly assumes that identical geometric approximation quality with respect to each party’s PF carries comparable evaluative significance across all DMs~\cite{KARSU2015343,ogryczak2014fair,yu2025towards}. However, in multi-party decision-making contexts, this assumption may not align with the notion of consensus. That is, a solution set may achieve excellent approximation for one party while providing only marginal benefits to others, yet still obtain a favorable averaged score. Consequently, mean-based metrics may fail to distinguish between outcomes exhibiting balanced relative gains and those biased toward a subset of parties.

%共识性问题
This limitation is further amplified by the solution structure of MPMOPs. Under strict Pareto rationality, solutions acceptable to all parties must lie in the intersection of their Pareto sets (PSs), which is often sparse or even empty. Classical metrics such as IGD and HV evaluate geometric approximation quality over the entire PF and may penalize solution sets concentrated in small common regions, even when such regions correspond to the only mutually acceptable outcomes. As a result, geometric balance or uniform coverage of the objective space, implicitly rewarded by these metrics, does not necessarily reflect consensus-oriented fairness in MPMOPs.

%全篇思路
Given that cooperation in MPMOPs cannot be adequately characterized by strictly common Pareto-optimal solutions alone, a more general notion of consensus is required. In this work, consensus refers to solution sets that yield mutually acceptable relative gains across DMs, rather than strictly dominating outcomes under classical Pareto rationality. From a cooperative bargaining perspective, each DM is associated with a bargaining region representing the set of solutions it considers acceptable. Rather than treating the existence of strictly common solutions as a prerequisite, we allow parties to tolerate controlled concessions, thereby inducing a unified notion of consensus solutions that subsumes strictly common Pareto-optimal solutions as a special case. Such compromises are quantified through a concession rate vector, which characterizes the acceptable degree of deviation for each DM.

%贡献
% Contributions
Building on this, we integrate classical performance metrics into a Nash-product-based evaluation framework so that solution sets yielding balanced relative gains across parties are consistently favored. To avoid superficially balanced yet non-consensual outcomes, we further introduce a penalty mechanism for solutions lying outside the jointly acceptable consensus region. The proposed framework is formulated as a performance evaluation methodology and is independent of specific optimization algorithms. The main contributions of this paper are summarized as follows:

\begin{itemize}
	\item \textit{Fairness limitations of existing performance metrics and evaluation axioms.}
	We identify fundamental limitations of existing mean-based performance metrics for MPMOPs, showing that they may yield biased or incomplete evaluation by failing to capture consensus and balanced relative gains among DMs. From a cooperative game-theoretic perspective, we formalize four axioms that a fairness-aware performance evaluation function for MPMOPs should satisfy: (A1) Pareto Monotonicity, (A2) Symmetry, (A3) Balance Preference, and (A4) Acceptability Monotonicity.
	
	\item \textit{Consensus optimal solutions and Nash-product-based evaluation framework.}
	By introducing a concession rate vector to quantify acceptable compromises for each DM, we generalize the classical notion of MPMOP optimal solutions and define a unified concept of consensus solutions. Building on this formulation, we propose a Nash-product-based evaluation framework that integrates classical performance metrics and theoretically prove that it satisfies all four axioms. The proposed framework provides a principled mechanism for jointly evaluating efficiency and fairness across heterogeneous DMs.
	
	\item 
	\textit{Benchmark extensions for consensus-deficient negotiation structures.}
	To enable more rigorous and representative testing of multi-party multi-objective optimization algorithms, we extend existing MPMOP benchmark suites by constructing problem instances with highly conflicting negotiation structures among DMs. These benchmarks exhibit empty intersections of party-specific PSs and therefore pose significant challenges for algorithmic convergence toward cooperative outcomes. They also provide an experimental basis for examining how different performance evaluation criteria evaluate algorithmic behavior under consensus-oriented settings.
	
\end{itemize}

%Organization structure
The rest of this paper is organized as follows. Section \ref{sec:preliminaries} introduces the preliminaries. 
Section \ref{sec:limitations} analyzes the limitations of prevailing performance metrics for MPMOPs. Section \ref{sec:proposed} presents the proposed fairness-aware performance evaluation framework. Section \ref{sec:experiments} reports experimental results comparing the proposed metric with meanIGD across multiple representative algorithms on two benchmarks. Section \ref{sec:discussion} discusses the validity of the proposed evaluation function and its potential extension to problems without a known PF. Finally,  Section \ref{sec: Conclusion} concludes the paper and outlines directions for future work.

\section{Preliminaries}
\label{sec:preliminaries}

This section introduces the basic concepts required for the subsequent analysis, including MPMOPs, Pareto dominance relations in multi-party settings, the Nash bargaining framework, and commonly used performance metrics for evaluating multi-party multi-objective evolutionary algorithms.

\subsection{Multi-party Multi-objective Optimization Problem}

MPMOPs extend classical multi-objective optimization problems by considering multiple DMs, each associated with an individual set of potentially conflicting objectives. A minimization MPMOP can be formally defined as follows.

% Definition of MPMOP
\begin{definition}[MPMOP~\cite{liu2020evolutionary}]
	\label{def:MPMOP}
	An MPMOP involves $M$ DMs. For each DM $m \in \{1,\dots,M\}$, let
	\[
	F_m(\mathbf{x}) = \big(f_{m,1}(\mathbf{x}),\dots,f_{m,K_m}(\mathbf{x})\big)
	\]
	denote its $K_m$-dimensional objective vector. The overall problem is defined as
	\[
	\min_{\mathbf{x} \in \mathcal{X}}
	F(\mathbf{x}) = \big(F_1(\mathbf{x}),\dots,F_M(\mathbf{x})\big),
	\]
	where $\mathbf{x} \in \mathcal{X} \subseteq \mathbb{R}^n$ is the decision vector and $\mathcal{X}$ denotes the feasible decision space.
\end{definition}

To evaluate solution quality from the perspective of each individual DM, Pareto dominance is recalled for party-wise objective vectors.

% Definition of party-wise pareto dominance
\begin{definition}[Party-wise Pareto Dominance]
	\label{def:pareto_dominance}
	For a given party $m$, let $F_m(\mathbf{x}) = (f_{m,1}(\mathbf{x}),\dots,f_{m,K_m}(\mathbf{x}))$. For two solutions $\mathbf{x}, \mathbf{y} \in \mathcal{X}$:
	\begin{enumerate}
		\item $\mathbf{x}$ weakly dominates $\mathbf{y}$ with respect to party $m$, denoted $\mathbf{x} \preceq_m \mathbf{y}$, if $f_{m,k}(\mathbf{x}) \le f_{m,k}(\mathbf{y})$ for all $k \in \{1,\dots,K_m\}$;
		\item $\mathbf{x}$ dominates $\mathbf{y}$ with respect to party $m$, denoted $\mathbf{x} \prec_m \mathbf{y}$, if $\mathbf{x} \preceq_m \mathbf{y}$ and there exists at least one $k \in \{1,\dots,K_m\}$ such that $f_{m,k}(\mathbf{x}) < f_{m,k}(\mathbf{y})$.
	\end{enumerate}
	
	Based on the above party-wise dominance relation, the party-wise Pareto-optimal set of DM $m$ is defined as 
		\begin{equation*}
			P_m^*
			=
			\left\{
			\mathbf{x}\in\mathcal{X}
			\mid
			\nexists\,\mathbf{y}\in\mathcal{X}
			\text{ such that }
			\mathbf{y}\prec_m\mathbf{x}
			\right\}.
		\end{equation*}
		That is, $P_m^*$ contains all feasible solutions that are not dominated
		from the perspective of DM $m$. The collection of objective values corresponding to $P_m^*$ is called Pareto set.
\end{definition}

Based on party-wise dominance relations, a joint dominance notion across all DMs can be defined.

% Definition of Multi-party Pareto Dominance
\begin{definition}[Multi-party Pareto Dominance]
	\label{def:multi_party_pareto_dominance}
	For two solutions $\mathbf{x}, \mathbf{y} \in \mathcal{X}$ in an MPMOP:
	\begin{enumerate}
		\item $\mathbf{x}$ weakly dominates $\mathbf{y}$ in the multi-party sense, denoted $\mathbf{x} \preceq^{\mathrm{MP}} \mathbf{y}$, if $\mathbf{x} \preceq_m \mathbf{y}$ holds for all $m \in \{1,\dots,M\}$;
		\item $\mathbf{x}$ dominates $\mathbf{y}$ in the multi-party sense, denoted $\mathbf{x} \prec^{\mathrm{MP}} \mathbf{y}$, if $\mathbf{x} \preceq_m \mathbf{y}$ holds for all $m \in \{1,\dots,M\}$ and there exists at least one party $m$ such that $\mathbf{x} \prec_m \mathbf{y}$.
	\end{enumerate}
	
	Following the common Pareto optimality concept in MPMOPs, the strict Multi-party Pareto-optimal set is defined as 
	\[ 
	P_{\mathrm{C}}^* = \bigcap_{m=1}^{M} P_m^*.
	\] 
	A solution $\mathbf{x}^*\in P_{\mathrm{C}}^*$ is called common Pareto-optimal solution if it is Pareto-optimal for every DM simultaneously and therefore $P_{\mathrm{C}}^*$ represents the common Pareto-optimal set.
\end{definition}

\subsection{Nash Bargaining Framework}

The Nash bargaining problem~\cite{nash1950bargaining} is a canonical cooperative game-theoretic model for analyzing negotiated outcomes among multiple participants. Let $\mathcal{S} \subseteq \mathbb{R}^M$ denote the feasible utility set, where each vector $\mathbf{u} = (u_1,\ldots,u_M)$ represents the utilities simultaneously attainable by all $M$ parties through cooperation. Let $\mathbf{d} = (d_1,\ldots,d_M)$ denote the disagreement (or fallback) utility vector, specifying the utility received by each party if no agreement is reached. The Nash bargaining problem is thus defined as the pair $(\mathcal{S}, \mathbf{d})$, subject to the feasibility condition
\[
\mathcal{S} \cap \{\mathbf{u} : \mathbf{u} \ge \mathbf{d}\} \neq \emptyset.
\]
To ensure the existence and uniqueness of the Nash bargaining solution, $\mathcal{S}$ is typically assumed to be closed and convex.

Within multi-objective game environments, cooperative and bargaining-based solution concepts have been extensively studied to characterize rational and fair outcomes~\cite{moulin1984implementing,cha2022ecs}. Existing approaches can be broadly categorized as follows:  
(a) extensions of Nash equilibrium to vector-valued payoffs, which replace scalar utility comparisons with Pareto dominance or aggregated utility mechanisms~\cite{ropke2022nash,li2024open,ruadulescu2020utility,ruadulescu2022opponent};  
(b) multi-objective bargaining solutions, which employ constructs such as the Nash product to formalize equitable compromise in vector payoff spaces~\cite{ruan2023multiobjective,li2022coordination,wang2023nash,meng2024nash}; and  
(c) preference-aware optimization frameworks, which explicitly incorporate heterogeneous preferences or individual utility regions into the decision-making process~\cite{ip2025user,huang2024direct,li2023interactive}.

While these studies establish theoretical foundations for fairness and cooperation, fairness is predominantly embedded in solution construction or equilibrium definitions. In contrast, existing performance metrics for MPMOPs are typically inherited from single DM multi-objective optimization and lack mechanisms for evaluating balanced cooperation and consensus quality across heterogeneous DMs.

\subsection{Performance Metrics}
\label{subsec:performance_metrics}

Performance metrics play a central role in evaluating multi-objective optimization algorithms by quantifying the quality of obtained solution sets. Among the most widely adopted metrics are the IGD and HV, which are commonly used to evaluate convergence and diversity in classical multi-objective optimization. In MPMOPs, however, each DM is associated with an individual PF, which complicates direct application of classical metrics. As a result, existing studies have proposed multi-party extensions of IGD and HV by aggregating party-wise performance evaluation.

\subsubsection{Inverted Generational Distance (IGD)}

The IGD evaluates how well an approximate solution set $P$ represents a reference $\mathrm{PF}$ by measuring the average distance from each point in $\mathrm{PF}$ to its nearest point in $P$:
\begin{equation}
	\label{def:IGD}
	\mathrm{IGD}(\mathrm{PF},P) = \frac{1}{|\mathrm{PF}|}\sum_{\mathbf{v} \in \mathrm{PF}} d(\mathbf{v},P),
\end{equation}
where $d(\mathbf{v},P)$ denotes the minimum Euclidean distance between $\mathbf{v}$ and any point in $P$. Smaller IGD values indicate closer approximation of the reference PF.

For MPMOPs, since each DM $m$ is associated with an individual $\mathrm{PF}_m$, Liu et al.~\cite{liu2020evolutionary} proposed a multi-party extension of IGD that aggregates distances across all parties. Specifically, the distance between a reference point $\mathbf{v}$ and a candidate solution set $P$ is defined as
\begin{equation}
		d(\mathbf{v}, P) = \min_{\mathbf{p} \in P} \sum_{m=1}^{M}
		\sqrt{\sum_{k=1}^{K_m} (v_{m,k} - p_{m,k})^2},
\end{equation}
where $K_m$ denotes the number of objectives of party $m$, and $v_{m,k}$ and $p_{m,k}$ represent the $k$-th objective value of party $m$ for $\mathbf{v}$ and $\mathbf{p}$, respectively. The remaining computation follows the classical IGD formulation. This multi-party IGD measures the overall geometric proximity of $P$ to all party-wise PFs.

\subsubsection{Hypervolume (HV)}

The HV metric measures the volume of the objective space dominated by a solution set $P$ with respect to a predefined reference point $\mathbf{r}$. For a minimization problem associated with party $m$, each solution $\mathbf{p} \in P$ is represented in the corresponding objective space as
\[
\mathbf{p} = (p_{m,1},\dots,p_{m,K_m}).
\]
The HV of $P$ with respect to party $m$ is defined as
\begin{equation}
	\label{def:HV}
	\mathrm{HV}_m(P)
	= \lambda\!\left(
	\bigcup_{\mathbf{p} \in P}
	[p_{m,1}, r_{m,1}] \times \cdots \times [p_{m,K_m}, r_{m,K_m}]
	\right),
\end{equation}
where $\lambda(\cdot)$ denotes the Lebesgue measure.

In MPMOPs, a unified joint hypervolume across all parties is generally unavailable due to heterogeneous objective spaces. A common practice is therefore to evaluate the hypervolume separately for each party and then aggregate the results. For example, Chen et al.~\cite{chen2024evolutionary} proposed computing the mean hypervolume as
\begin{equation}
	\mathrm{meanHV}(P) = \frac{1}{M}\sum_{m=1}^M \mathrm{HV}_m(P),
\end{equation}
which treats the party-wise HV values symmetrically.

\subsubsection{Concatenated-objective evaluation}
Another possible evaluation strategy is to concatenate all objectives of all DMs into:
\[
F_{\mathrm{all}}(x)=
(f_{1,1}(x),\ldots,f_{1,K_1}(x),\ldots,
f_{M,1}(x),\ldots,f_{M,K_M}(x)).
\]
Classical indicators can then be computed in the resulting \(K=\sum_{m=1}^{M}K_m\)-dimensional objective space using the standard IGD and HV definitions in (1) and (3). This strategy provides a useful geometric indicator over all objectives. However, it treats an MPMOP as a conventional many-objective optimization problem and does not explicitly preserve the party structure. Thus, it may not reveal whether a solution set is balanced across DMs or biased toward a subset of parties.

Multiple indicators can provide complementary information because different metrics emphasize different aspects of approximation quality~\cite{wang2021statistical}. Thus, multi-indicator analysis can reduce the risk of drawing conclusions from a single metric. Under the proposed framework, one possible extension is to compute a Nash-product score for each selected indicator, and then conduct joint analysis on the resulting scores. This preserves party-wise fairness aggregation while enabling multi-indicator comparison.

\section{Fairness Limitations of Classical Metrics}
\label{sec:limitations}
In MPMOPs, classical performance metrics such as IGD and HV are typically evaluated with respect to each DM’s PF and then aggregated across parties using arithmetic means, yielding metrics such as $\mathrm{meanIGD}$ and $\mathrm{meanHV}$. While computationally convenient, this practice implicitly treats the problem as if it involved a single representative DM, thereby overlooking fairness considerations among heterogeneous DMs.

As a consequence, the aggregated metric value may improve even when one or more DMs experience substantial utility degradation. Such evaluations are inconsistent with classical cooperative bargaining principles, in which balanced concessions and symmetry among participants are central. From a negotiation perspective, each DM evaluates the same solution set through a distinct utility mapping~\cite{liu2020evolutionary,taherdoost2023multi,zhang2021personalized}. Simple averaging of per-party IGD or HV values can therefore obscure asymmetric sacrifices, assigning similar scores to fundamentally different cooperative outcomes.

\begin{figure}[ht]
	\centering
	\small
	
	\subfloat[Decision spaces.
	Left: the PSs of the two DMs are parallel and disjoint, indicating the absence of a common Pareto-optimal region. In this case, solution set $P_2$ is strongly biased toward DM~A, whereas $P_1$ exhibits a more symmetric compromise.
	Right: the PSs of the two DMs intersect, forming a common Pareto-optimal region. Here, $P_1$ concentrates around this intersection, while $P_2$ is biased toward DM~B.]
	{\includegraphics[width=1\columnwidth]{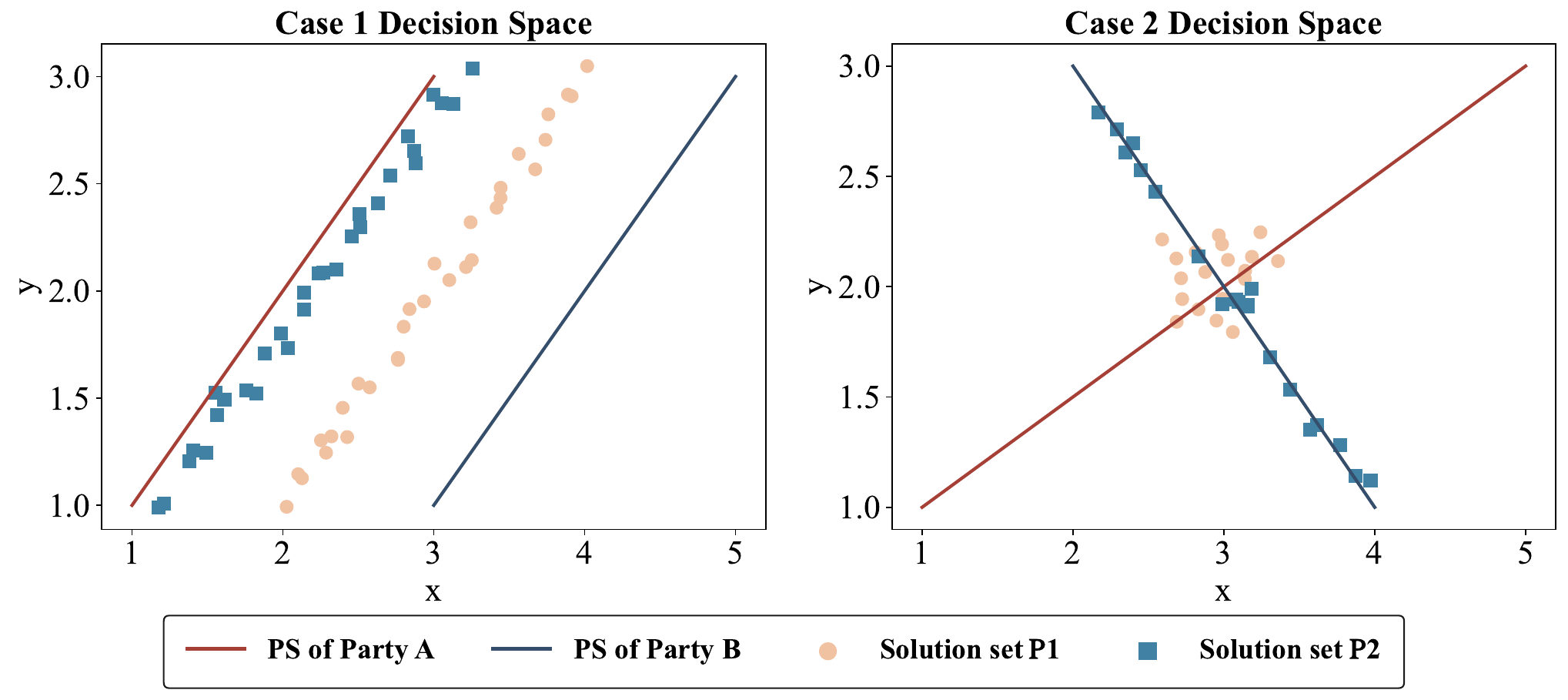}}
	
	\subfloat[Objective spaces.
	The first row corresponds to Case~1, and the second row corresponds to Case~2.
	In Case~1, $P_2$ closely approximates the PF of DM~A but remains far from that of DM~B, whereas $P_1$ maintains a more balanced approximation for both DMs.
	In Case~2, $P_2$ provides relatively uniform coverage of DM~B’s PF, while $P_1$ concentrates near the intersection of the two PFs, reflecting a more consensus-oriented outcome.]
	{\includegraphics[width=1\columnwidth]{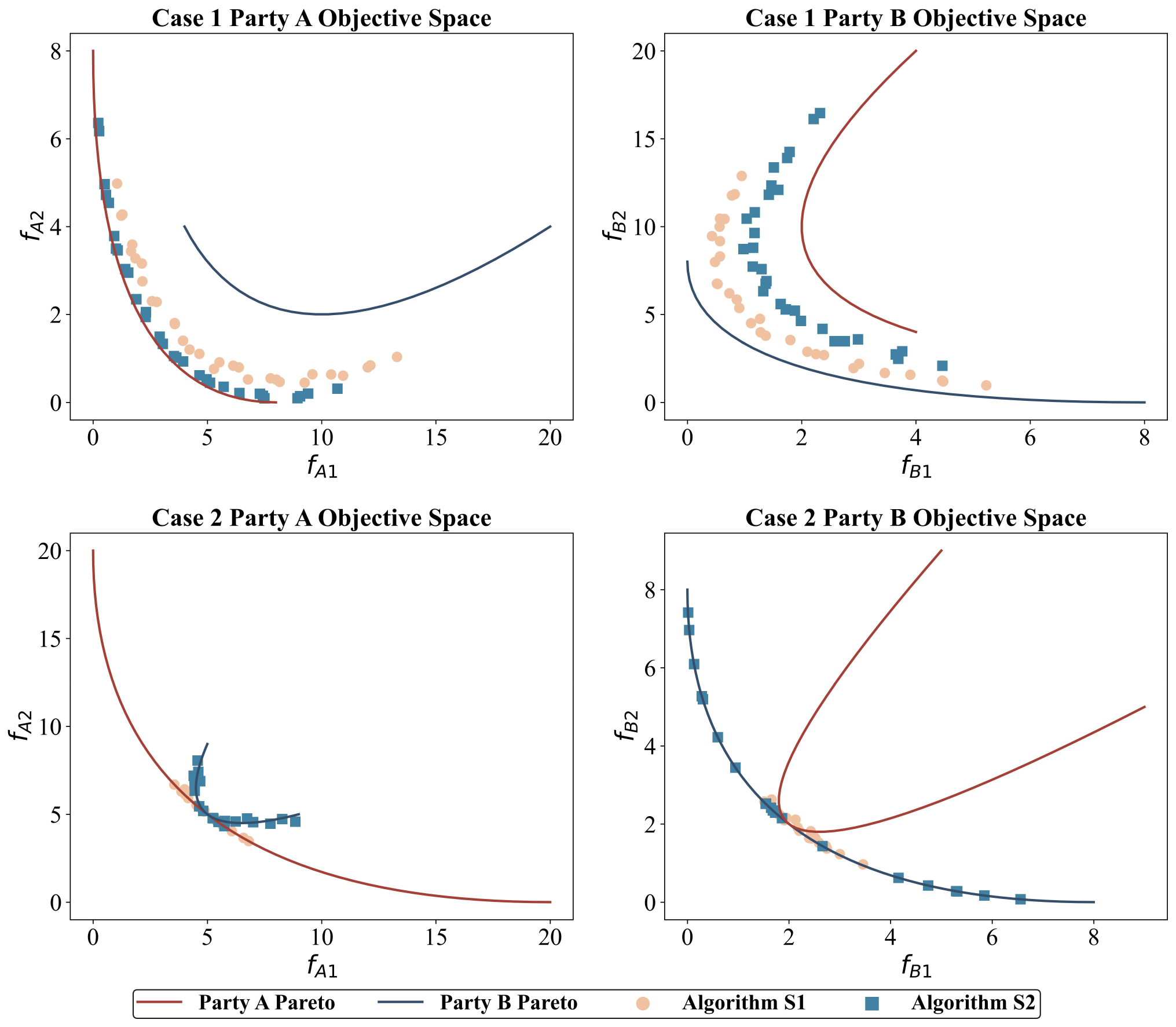}}
	
	\caption{Visualization of fairness limitations of mean-based performance metrics in representative MPMOP cases. Decision and objective spaces for two representative MPDMP cases are shown, illustrating scenarios with and without a common Pareto-optimal region.}
	\label{fig:Nash Dilemma}
\end{figure}

To make this limitation explicit, we conduct visualization-based comparative analyses on multi-party distance minimization problems (MPDMPs)~\cite{she2023multiparty}. These problems exhibit a transparent geometric structure that allows simultaneous inspection of solution behavior in both decision and objective spaces, making them particularly suitable for illustrating how solution sets relate to different PFs and whether their trade-offs align with intuitive notions of cooperative fairness. MPDMPs model scenarios in which multiple DMs share a common decision space but evaluate solutions using distinct distance-based objectives reflecting heterogeneous preferences. We consider two DMs, denoted as DM~A and DM~B, sharing a common decision variable $\mathbf{x} \in \mathbb{R}^2$. Each DM seeks to minimize Euclidean distances from $\mathbf{x}$ to two distinct target points. Two representative cases are constructed.

\textit{Case~1 (Disjoint Pareto-optimal regions):}  
DM~A minimizes distances to $(1,1)$ and $(3,3)$, while DM~B minimizes distances to $(3,1)$ and $(5,3)$. The resulting PSs in the decision space are parallel and disjoint, implying that no solution simultaneously lies on both DMs’ PSs.

\textit{Case~2 (Intersecting Pareto-optimal regions):}  
DM~A minimizes distances to $(1,1)$ and $(5,3)$, while DM~B minimizes distances to $(2,3)$ and $(4,1)$. In this configuration, the two PSs intersect at $\mathbf{x}^*=(3,2)$, yielding a common Pareto-optimal region that naturally represents a balanced compromise.

Fig.~\ref{fig:Nash Dilemma} visualizes both cases in the decision and objective spaces, together with two representative solution sets, denoted as $P_1$ and $P_2$.

\begin{table}[htbp]
	\centering
	\caption{IGD and HV comparison for representative solution sets in Case 1 and Case 2.}
	\label{tab:case1_case2_comparison}
	\small
	\setlength{\tabcolsep}{4pt}
	\begin{tabular}{lcccccc}
		\toprule
		& $\mathrm{IGD_A}$ & $\mathrm{IGD_B}$ & $\mathrm{meanIGD}$ & $\mathrm{HV_A}$ & $\mathrm{HV_B}$ & $\mathrm{meanHV}$ \\
		\midrule
		
		\multicolumn{7}{c}{\textbf{Case 1}} \\
		$P_1$ & 0.73 & 0.73 & 0.73 $\approx$ & 3.90 & 3.90 & 3.90 $\approx$ \\
		$P_2$ & 0.14 & 1.36 & 0.75 $\approx$ & 5.42 & 2.06 & 3.74 $\approx$ \\
		
		\midrule
		\multicolumn{7}{c}{\textbf{Case 2}} \\
		$P_1$ & 5.18 & 1.78 & 3.48 $\downarrow$ & 9.65 & 5.19 & 7.42 $\approx$ \\
		$P_2$ & 5.42 & 0.40 & 2.91 $\uparrow$ & 8.19 & 6.68 & 7.43 $\approx$ \\
		
		\bottomrule
	\end{tabular}
\end{table}

Table~\ref{tab:case1_case2_comparison} reports the IGD and HV values for both cases. In Case~1, solution set $P_2$ strongly favors DM~A, achieving a very small $\mathrm{IGD_A}$ at the expense of a large $\mathrm{IGD_B}$. However, meanIGD assigns nearly identical scores to $P_1$ and $P_2$, effectively masking this imbalance. A similar effect is observed for HV. From a bargaining perspective, $P_1$ corresponds to a more symmetric compromise, whereas $P_2$ does not satisfy the symmetry principle central to cooperative bargaining.

In Case~2, where a common Pareto-optimal region exists, $P_1$ concentrates near the intersection of the two PSs and represents a balanced cooperative outcome. Nevertheless, meanIGD slightly favors $P_2$, which is clearly biased toward DM~B, due to its uneven distance profile. meanHV again fails to distinguish the two solution sets. This behavior is inconsistent with the bargaining intuition that solution sets close to the common region should be preferred as fair and effective cooperative outcomes.

These examples demonstrate that mean-based performance metrics not only mask severe imbalance when no common Pareto-optimal region exists, but also fail to appropriately favor solution sets that concentrate near the common Pareto-optimal region when such solutions are available.

\section{A Fairness-Aware Nash-Product-Based Evaluation Framework for MPMOPs}
\label{sec:proposed}
%对整节内容的简单介绍
This section develops a fairness-aware evaluation method for MPMOPs. Motivated by the limitations of mean-based performance aggregation, we propose an evaluation function that explicitly accounts for balanced compromise, impartiality across DMs, and sensitivity to mutual acceptability. Our construction proceeds in three steps. First, we formalize a set of axioms that capture rationality, symmetry, equity, and acceptability requirements for evaluating solution sets in MPMOPs. Second, we show that these axioms are satisfied by a \emph{Nash-product-based evaluation function}, which aggregates per-DM utilities multiplicatively. Third, we extend this evaluation function to explicitly penalize violations of inter-party acceptability through a concession-based penalty mechanism that augments individual loss terms. These steps yield a unified and axiomatically grounded evaluation framework, whose theoretical properties are formally established in the remainder of this section.

\subsection{Axioms for Fairness-Aware Evaluation}
\label{sec:axioms}

We consider an MPMOP involving $M$ DMs, indexed by $\mathcal{M}=\{1,\dots,M\}$. Each DM evaluates a candidate solution set $P$ using a common performance metric, interpreted from its own perspective. Let $u_m(P)\in\mathbb{R}_{+}$ denote the utility for DM $m$, where larger values indicate better performance. An overall evaluation function aggregates individual utilities as
\[
\Psi(P) = \Psi\big(u_1(P),\dots,u_M(P)\big).
\]

The following axioms characterize structural properties that such an aggregation should satisfy in order to reflect fairness considerations inherent in MPMOPs.

\begin{axiom}[Pareto Monotonicity~\cite{nash1950bargaining}]
	\label{ax:pareto}
	For any two solution sets $P_1$ and $P_2$, if 
	\[
	u_i(P_1) \ge u_i(P_2) \ \forall i \in \mathcal{M}, 
	\]
	and there exists $j \in \mathcal{M}$ such that $u_j(P_1) > u_j(P_2)$, then
	\[
	\Psi(P_1) > \Psi(P_2).
	\]
\end{axiom}

\begin{axiom}[Symmetry~\cite{nash1950bargaining}]
	\label{ax:symmetry}
	For any permutation $\pi$ of $\mathcal{M}$, assuming that all DM-specific parameters are permuted consistently with their corresponding utilities,
	\[
	\Psi\big(u_1(P),\dots,u_M(P)\big) = \Psi\big(u_{\pi(1)}(P),\dots,u_{\pi(M)}(P)\big).
	\]
\end{axiom}

\begin{axiom}[Balance Preference~\cite{atkinson1970measurement}]
	\label{ax:balance}
	Let $\mathbf{u}(P_1)$ and $\mathbf{u}(P_2)$ denote the utility vectors of two solution sets such that
	\[
	\sum_{m=1}^{M} u_m(P_1) = \sum_{m=1}^{M} u_m(P_2).
	\]
	If $\mathbf{u}(P_1)$ is obtained from $\mathbf{u}(P_2)$ by a mean-preserving contraction, then
	\[
	\Psi(P_1) > \Psi(P_2).
	\]
\end{axiom}

\begin{axiom}[Acceptability Monotonicity]
	\label{ax:consensus}
	Let $\mathcal{A} \subseteq \mathcal{X}$ denote the common acceptable region. 
	For any two solution sets $P_1, P_2 \subseteq \mathcal{X}$ such that
	\[
	P_2 = P_1 \cup \{\mathbf{x}'\} \setminus \{\mathbf{x}\}, \quad \mathbf{x} \in P_1 \cap \mathcal{A}, \quad \mathbf{x}' \notin \mathcal{A},
	\]
	then
	\[
	\Psi(P_1) > \Psi(P_2).
	\]
\end{axiom}

Collectively, these axioms encode rationality (Pareto Monotonicity), impartiality across DMs (Symmetry), preference for equitable compromise under fixed total utility (Balance Preference), and strict sensitivity to violations of mutual acceptability (Acceptability Monotonicity). Axioms 1-3 are adapted from established bargaining and social-welfare principles: Pareto Monotonicity and Symmetry characterize efficiency improvement and impartial treatment among DMs, whereas Balance Preference captures the preference for a more equitable utility distribution under a fixed total utility. In contrast, Axiom 4 is introduced specifically for MPMOP evaluation. It requires the evaluation function to explicitly penalize solution sets that violate the jointly acceptable region, thereby preventing apparently balanced but non-consensual outcomes from receiving favorable scores.

\subsection{Nash-Product-Based Evaluation Function}

The axioms impose strong structural constraints on $\Psi(\cdot)$. In particular, Pareto Monotonicity and Symmetry rule out purely ordinal or identity-dependent aggregations, while Balance Preference is incompatible with purely additive or mean-based aggregation schemes, which, as shown in Section~\ref{sec:limitations}, may assign similar scores to highly imbalanced outcomes. These requirements collectively make a multiplicative form of aggregation particularly suitable, due to its sensitivity to imbalance and proportional losses.

In MPMOP, cooperation typically does not generate surplus gains; instead, each DM incurs unavoidable performance losses relative to its individual PS. Fairness therefore corresponds to distributing these losses in a balanced manner. Inspired by the Nash bargaining principle~\cite{nash1950bargaining}, we adopt a multiplicative utility aggregation to evaluate fairness among DMs. In the proposed formulation, each DM's utility is constructed from its performance loss relative to its own PF, rather than from a predefined bargaining utility vector. The resulting Nash-product-based score evaluates whether an approximation set distributes performance degradation in a balanced manner across heterogeneous DMs, thereby adapting the bargaining principle to party-wise solution-set evaluation in MPMOPs.

Let $L_m(P) \ge 0$ denote the performance loss of solution set $P$ for DM $m$, where smaller values indicate better approximation quality with respect to DM $m$'s PF. At this stage, $L_m(P)$ only evaluates party-wise approximation performance and does not include the concession-based penalty introduced later. We define the utility of DM $m$ as
\begin{equation}
	u_m(P) = C -L_m(P).
\end{equation}
$C$ is introduced solely to ensure positivity of utilities here, and does not affect the relative comparison between solution sets.

Following the Nash bargaining, we aggregate individual utilities multiplicatively and define the evaluation function as
\begin{equation}
	\Psi_\mathrm{NP}(P) = \prod_{m=1}^M u_m(P) = \prod_{m=1}^M \left( C -L_m(P)\right).
\end{equation}

This function penalizes disproportionate losses multiplicatively and strictly favors balanced compromises when total utility is fixed.

\subsection{Consensus Solutions via Concession-Based Penalty}
%退让度的引入动机
If a strictly common Pareto-optimal solution exists, it represents the most stringent form of cooperation in a MPMOP, and an effective algorithm is expected to approximate it as closely as possible. However, such solutions may not always exist, and even when they do, fairness alone does not guarantee joint acceptability. In particular, balanced performance improvements may still impose unacceptable sacrifices on individual parties. This makes it necessary to explicitly characterize and quantify consensus in multi-party settings.

In many practical MPMOPs, the PSs of different DMs exhibit little or no overlap. In such cases, cooperation can only be achieved if each party concedes part of its individually optimal outcome, allowing a mutually acceptable region to emerge. Moreover, cooperative outcomes typically involve explicitly quantifiable sacrifices across DMs, reflecting differences in how far each party deviates from its own PS.

In MPMOPs, fairness aggregation alone is insufficient because each DM has its own objective space and PS, and the strict common PS may be empty. To characterize mutual acceptability beyond strict common optimality, we introduce concession rates to measure each DM's deviation from its own PS. These rates induce a mutually acceptable region and provide the basis for penalizing non-consensus solutions. To formally characterize such sacrifices, we introduce the notion of a \emph{concession rate}, which measures the relative performance loss tolerated by each DM with respect to its own PS.

\begin{definition}[Concession Rate] 
	\label{def:concession_rate} 
		For each DM $m$, the concession rate $\gamma_m(\mathbf{x})$ measures the relative deviation of a solution $\mathbf{x}$ from the PS $\mathcal{P}_m^{*}$ of DM $m$. Specifically, let $D_m(\mathbf{x})$ denote the normalized distance from $\mathbf{x}$ to the closest solution in $\mathcal{P}_m^{*}$, and let $\Delta_m$ denote the maximum value of this distance among the Pareto-optimal solutions of all other DMs. Then, 
		\begin{equation} 
			\gamma_m(\mathbf{x}) = \frac{D_m(\mathbf{x})}{\Delta_m} \geq 0.
		\end{equation} 
		Intuitively, $\gamma_m(\mathbf{x})=0$ means that $\mathbf{x}$ lies on $\mathcal{P}_m^{*}$ and requires no concession from DM $m$, whereas $\gamma_m(\mathbf{x})=1$ corresponds to the largest reference deviation among other DMs' Pareto-optimal solutions. Therefore, a larger $\gamma_m(\mathbf{x})$ indicates a larger concession required from DM $m$. 
\end{definition}

The concession rate quantifies the relative sacrifice tolerated by each DM in achieving cooperative outcomes. Unlike classical $\varepsilon$-based metrics in multi-objective optimization, which measure dominance relaxation or approximation quality~\cite{kollat2005value}, $\gamma_m$ captures inter-party compromise across heterogeneous PSs.

Using the concession rate, we extend the classical notion of multi-party Pareto optimal solutions to explicitly incorporate acceptability constraints.

\begin{definition}[Generalized Multi-party Pareto-optimal Solutions]
	\label{def:generalized_mpmop_optima}
	Let $\mathcal{P}^* \subset \mathcal{X}$ denote the set of solutions that are non-dominated with respect to the joint objective vector
	\[
	F(\mathbf{x}) =
	(f_{1,1}(\mathbf{x}), \dots, f_{1,K_1}(\mathbf{x}), \dots,
	f_{M,1}(\mathbf{x}), \dots, f_{M,K_M}(\mathbf{x})).
	\]
	
	Given concession thresholds $\hat{\gamma}_m \in [0,1]$ for each DM $m$, define the extended acceptable set
	\[
	\mathcal{A}_m(\hat{\gamma}_m)
	=
	\{\mathbf{x} \in \mathcal{P}^* \mid \gamma_m(\mathbf{x}) \le \hat{\gamma}_m\}.
	\]
	
	The mutually acceptable region is defined as
	\[
	\mathcal{A}
	=
	\bigcap_{m=1}^{M} \mathcal{A}_m(\hat{\gamma}_m).
	\]
\end{definition}

If $\hat{\gamma}_m = 0$ for all $m$, the mutually acceptable region reduces to the classical strictly common PS. If $\mathcal{A} = \emptyset$, no solution satisfies all DMs' concession thresholds simultaneously. Otherwise, each $\mathbf{x} \in \mathcal{A}$ represents a consensus solution that balances efficiency, fairness, and acceptability. Fig.~\ref{fig:Common acceptable set} visualize the concept of concession rates and mutually acceptable region.

\begin{figure}
	\centering
	\includegraphics[width=1\linewidth]{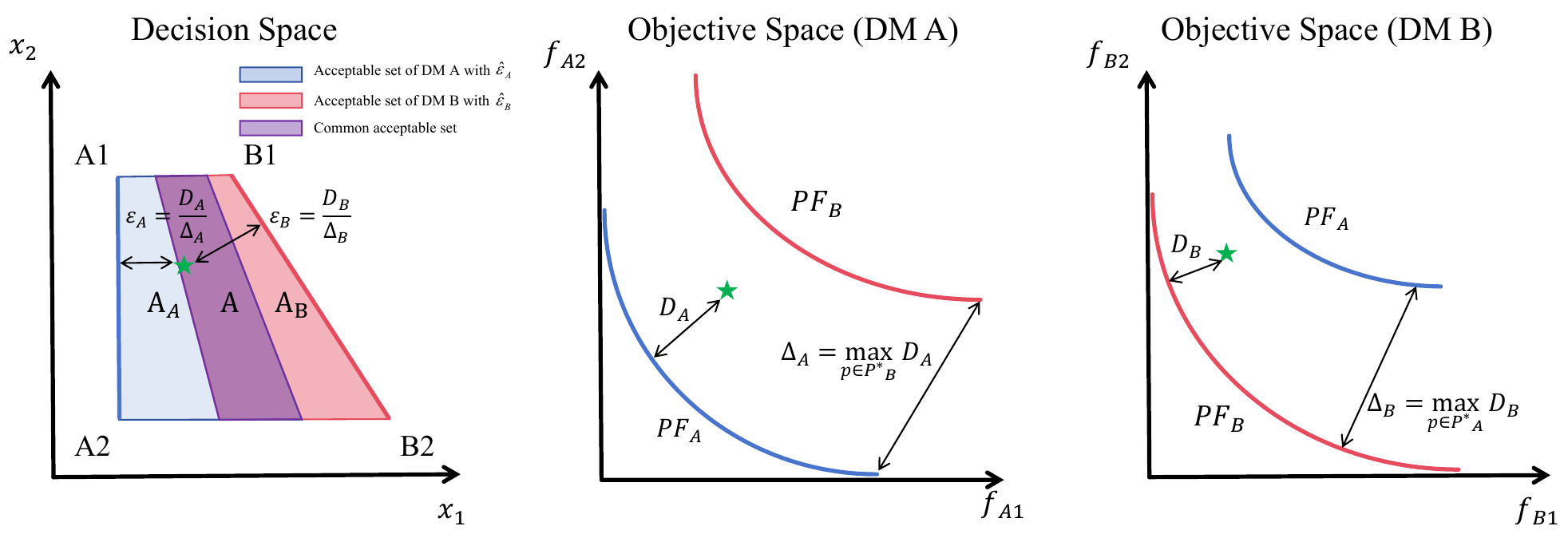}
	\caption{Illustration of concession rates and the mutually acceptable region. $\gamma_m$ quantifies the normalized deviation from it. The intersection $\mathcal{A}$ indicates solutions meeting all DMs' concession thresholds.}
	\label{fig:Common acceptable set}
\end{figure}

Having characterized the mutually acceptable region induced by concession constraints, we now incorporate acceptability into the Nash-product-based evaluation function $\Psi_{\mathrm{NP}}$, which introduced earlier satisfies Pareto monotonicity, symmetry, and balance preference, but does not explicitly penalize violations of acceptability.

\begin{definition}[Non-consensus Penalty]
		\label{def:non_consensus_penalty} 
		Let $\hat{\gamma}_m\in[0,1]$ denote the concession threshold of DM $m$. For a solution $\mathbf{x}$, the acceptability violation of DM $m$ is defined as 
		\begin{equation} 
			\phi_m(\mathbf{x}) = \max\{0,\,\gamma_m(\mathbf{x})-\hat{\gamma}_m\}. 
		\end{equation} 
		If $\gamma_m(\mathbf{x})\leq \hat{\gamma}_m$, then $\mathbf{x}$ is acceptable to DM $m$ and no penalty is incurred. Otherwise, $\phi_m(\mathbf{x})$ measures how much the concession required by $\mathbf{x}$ exceeds the tolerance of DM $m$. For a solution set $P\subset\mathcal{X}$, the non-consensus penalty of DM $m$ is accumulated over all solutions in $P$: 
		\begin{equation} 
			L_m^{\mathrm{pen}}(P) = \sum_{\mathbf{x}\in P} \phi_m(\mathbf{x}).
		\end{equation} 
		Therefore, $L_m^{\mathrm{pen}}(P)$ quantifies the total degree to which the solution set violates the acceptability requirement of DM $m$. 
	\end{definition}

Accordingly, the utility of DM $m$ is given by
\[
u_m(P)
=
C - (L_m(P) + \lambda_m L_m^{\mathrm{pen}}(P)),
\]
where $\lambda_m > 0$ controls the strictness of concession enforcement. The Nash-product evaluation function is redefined as
\[
\Psi_{\mathrm{NP}}(P)
=
\prod_{m=1}^{M}
\bigl(C - (L_m(P) + \lambda_m L_m^{\mathrm{pen}}(P))\bigr).
\]

This formulation preserves the structural properties of the original Nash-product-based criterion while explicitly discouraging non-consensus solutions. As a result, $\Psi_{\mathrm{NP}}$ provides a unified performance evaluation that jointly accounts for efficiency, fairness, and mutual acceptability in MPMOPs.

\subsection{Axiomatic Analysis of the Proposed Evaluation Criterion}

In this subsection, we show that the Nash-product-based formulation satisfies the axioms. The four axioms are verified separately through Lemmas~\ref{thm:pareto}--\ref{thm:acceptability}.

\begin{lemma}[Pareto Monotonicity]
	\label{thm:pareto}
	The Nash-product-based evaluation function $\Psi_{\mathrm{NP}}(P)$ satisfies Axiom~\ref{ax:pareto}.
\end{lemma}

\begin{proof}
	Consider two solution sets $P_1$ and $P_2$ such that $u_m(P_1) \ge u_m(P_2)$ for all $m$, and there exists at least one DM $j$ with $u_j(P_1) > u_j(P_2)$.  
	Because the product of positive numbers is strictly increasing in each component, we have
	\[
	\prod_{m=1}^{M} u_m(P_1) > \prod_{m=1}^{M} u_m(P_2).
	\]
	Hence, $\Psi_{\mathrm{NP}}(P_1) > \Psi_{\mathrm{NP}}(P_2)$, establishing Pareto monotonicity.
\end{proof}

\begin{lemma}[Symmetry]
	\label{thm:symmetry}
	The Nash-product-based evaluation function $\Psi_{\mathrm{NP}}(P)$ satisfies Axiom~\ref{ax:symmetry}.
\end{lemma}

\begin{proof}
	Let $\pi$ be any permutation of DM indices. Permuting both the utilities and all associated DM-specific parameters yields
	\[
	\prod_{m=1}^{M} u_{\pi(m)}(P) = \prod_{m=1}^{M} u_m(P),
	\]
	since multiplication is commutative. Therefore, the evaluation is invariant under any permutation of DMs, and symmetry is satisfied.
\end{proof}

\begin{lemma}[Balance Preference]
	\label{thm:balance}
	The Nash-product-based evaluation function $\Psi_{\mathrm{NP}}(P)$ satisfies Axiom~\ref{ax:balance}.
\end{lemma}

\begin{proof}
	Consider two solution sets $P_1$ and $P_2$ with equal total utility and such that $\mathbf{u}(P_1)$ is obtained from $\mathbf{u}(P_2)$ by a mean-preserving contraction.  
	The logarithm of the Nash product is strictly concave:
	\[
	\log \Psi_{\mathrm{NP}}(P_1) = \sum_{m=1}^{M} \log u_m(P_1).
	\]
	By the definition of mean-preserving contraction and Jensen's inequality, we have
	\[
	\sum_{m=1}^{M} \log u_m(P_1) > \sum_{m=1}^{M} \log u_m(P_2),
	\]
	and exponentiating both sides yields
	\[
	\Psi_{\mathrm{NP}}(P_1) > \Psi_{\mathrm{NP}}(P_2),
	\]
	as required.
\end{proof}

\begin{lemma}[Acceptability Monotonicity]
	\label{thm:acceptability}
	The Nash-product-based evaluation function $\Psi_{\mathrm{NP}}(P)$ satisfies Axiom~\ref{ax:consensus} for sufficiently large penalty coefficients $\{\lambda_m\}_{m=1}^M$.
\end{lemma}

\begin{proof}
	Consider two solution sets
	\[
	P_2 = P_1 \cup \{\mathbf{x}'\} \setminus \{\mathbf{x}\},
	\quad
	\mathbf{x} \in P_1 \cap \mathcal{A},
	\quad
	\mathbf{x}' \notin \mathcal{A}.
	\]
	
	By definition of $\mathcal{A}$, there exists at least one DM $m_0 \in \{1,\dots,M\}$ such that
	\[
	\phi_{m_0}(\mathbf{x}') = \max\{0, \gamma_{m_0}(\mathbf{x}') - \hat{\gamma}_{m_0}\} > 0.
	\]
	
	The utility of DM $m$ is
	\[
	u_m(P) = C - \big(L_m(P) + \lambda_m L_m^{\mathrm{pen}}(P)\big).
	\]
	
	Since $P_1$ and $P_2$ differ in exactly one solution, the losses change as
	\[
	L_{m_0}(P_2) = L_{m_0}(P_1) + \mu_{m_0}(\mathbf{x}'),
	\]
	\[
	L_{m_0}^{\mathrm{pen}}(P_2) = L_{m_0}^{\mathrm{pen}}(P_1) + \phi_{m_0}(\mathbf{x}').
	\]
	Hence,
	\begin{equation}
		\label{eq:u_m0_diff}
		u_{m_0}(P_2) = u_{m_0}(P_1) - \big(\mu_{m_0}(\mathbf{x}') + \lambda_{m_0} \phi_{m_0}(\mathbf{x}')\big).
	\end{equation}
	
	For all $m \neq m_0$, let $\mu_m(\mathbf{x}') \ge 0$ denote the maximal possible utility increase due to replacing $\mathbf{x}$ by $\mathbf{x}'$. Let $\underline{u}_m$ denote a positive lower bound of $u_m(P_1)$ ensured by the choice of $C$. Then
	\[
	u_m(P_2) \le u_m(P_1) + \mu_m(\mathbf{x}') \le \underline{u}_m + \mu_m(\mathbf{x}').
	\]
	
	By bounding each factor with the lower bound $\underline{u}_m$, we obtain
	\[
	\prod_{m \neq m_0} \left(1 + \frac{\mu_m(\mathbf{x}')}{u_m(P_1)}\right)
	\le \prod_{m \neq m_0} \left(1 + \frac{\mu_m(\mathbf{x}')}{\underline{u}_m}\right)
	= \alpha^{M-1}.
	\]
	
	Then a sufficient condition for strict decrease of Nash-product ratio is
	\[
	\frac{\Psi_{\mathrm{NP}}(P_2)}{\Psi_{\mathrm{NP}}(P_1)}
	=
	\left(1-\frac{\mu_{m_0}(\mathbf{x}') + \lambda_{m_0} \phi_{m_0}(\mathbf{x}')}{u_{m_0}(P_1)}\right)
	\alpha^{M-1} < 1.
	\]
	
	which yields
	\[
	\lambda_{m_0} > \frac{u_{m_0}(P_1) \big(1 - \alpha^{1-M}\big) - \mu_{m_0}(\mathbf{x}')}{\phi_{m_0}(\mathbf{x}')}.
	\]
	
	Therefore, choosing $\lambda_{m_0}$ sufficiently large ensures
	$\Psi_{\mathrm{NP}}(P_2) < \Psi_{\mathrm{NP}}(P_1)$.
\end{proof}

\begin{remark}
	Lemma~\ref{thm:acceptability} establishes a sufficient condition on the penalty coefficients $\lambda_m$ under which acceptability violations dominate any potential utility gains of other DMs.
	The bound is derived from a worst-case analysis and is not intended to be tight.
	In practice, moderate penalty values are sufficient to enforce acceptability, while the existence result provides theoretical consistency of the proposed evaluation function.
\end{remark}

The above result formally shows that the Nash-product-based evaluation function simultaneously captures efficiency improvement (Pareto monotonicity), fair treatment of DMs (symmetry), balanced compromise across parties (balance preference), and concession rationality for DMs (acceptability monotonicity), which are essential for evaluating solution sets in MPMOPs. Thus, we conclude that the proposed Nash-product-based evaluation framework identifies the utility-balanced solution set among all equally efficient outcomes. This provides a theoretically grounded and consistent measure of fairness in multi-party multi-objective optimization.

\begin{theorem}[Axiomatic Soundness of the Proposed Evaluation Framework]
	\label{thm:axiomatic_soundness}
	With sufficiently large penalty coefficients \(\{\lambda_m\}_{m=1}^{M}\), the fairness-aware Nash-product-based evaluation function
	\[
	\Psi_{\mathrm{NP}}(P)
	=
	\prod_{m=1}^{M}
	\bigl(C - (L_m(P) + \lambda_m L_m^{\mathrm{pen}}(P))\bigr).
	\]
	satisfies all four axioms introduced in Section~\ref{sec:axioms}, namely Pareto Monotonicity (A1), Symmetry (A2), Balance Preference (A3), and Acceptability Monotonicity (A4).
\end{theorem}

\begin{proof}
	Lemmas~\ref{thm:pareto}--\ref{thm:balance} establish that the Nash-product aggregation satisfies Axioms~(A1)--(A3). Lemma~\ref{thm:acceptability} further shows that the introduction of non-consensus penalties enforces Acceptability Monotonicity (A4). Therefore, the proposed evaluation function satisfies all axioms simultaneously.
\end{proof}
	
\section{Experiments}
\label{sec:experiments}

This section presents comprehensive experiments to demonstrate the effectiveness of the proposed fairness-aware evaluation framework. Two representative classes of multi-party optimization problems are considered. These experiments complement the theoretical analysis and evaluate the behavior of the proposed framework under different consensus conditions.

\subsection{Experimental Design and Setup}

\subsubsection{Test Problems}

Two categories of benchmark problems are considered. The first category consists of continuous multi-party multi-objective optimization problems (MPMOPs), including MPMOP1--MPMOP11 from~\cite{liu2020evolutionary} and six newly constructed benchmarks, MPMOP12--MPMOP17, designed to cover non-consensus scenarios. The second category focuses on multi-party distance minimization problems (MPDMPs), including MPDMP1--MPDMP12 from~\cite{she2023multiparty} and six additional instances, MPDMP13--MPDMP18. Detailed definitions of the extended benchmark problems are provided in the supplement document. For all benchmarks, each algorithm is independently run 30 times, and the mean and standard deviation of the results are reported. In terms of problem scale, the benchmark suite considered in this study contains two to three DMs. Each DM is associated with two to five objectives. The proposed evaluation framework itself is formulated for a general number of DMs and party-wise objectives, whereas the above range specifies the empirical scope examined in this paper.

\subsubsection{Compared Algorithms}

Several representative algorithms are selected for comparison.
For the MPMOP experiments, OptAll, which aggregates all DMs' objectives into a single multi-objective optimization problem, and OptMPNDS~\cite{liu2020evolutionary}, one of the earliest algorithms tailored for multi-party multi-objective optimization, are considered. For the MPDMP experiments, OptAll and OptMPNDS are further compared with two improved variants, OptMPNDS2~\cite{she2021new} and OptMPNDS3~\cite{she2023multiparty}. OptMPNDS2 extends OptMPNDS by incorporating non-dominated sorting information from all parties, while OptMPNDS3 further enhances the search process using a PF-based initialization strategy and an adaptive differential evolution operator. Among these methods, OptMPNDS3 has been reported to achieve strong performance on MPDMPs with existing common optimal solutions.

\subsubsection{Performance metrics}

All obtained solution sets are evaluated under two classes of metrics.
\begin{itemize}
	\item Classical metrics: meanIGD, computed as the arithmetic mean across all DMs.
	\item Fairness-aware metric: the proposed Nash-product-based evaluation function $\Psi_{\mathrm{NP}}$, where larger values indicate better-balanced inter-party compromise.
\end{itemize}

\subsubsection{Parameter Settings}

To ensure a fair comparison, all algorithms are implemented under identical experimental settings. The penalty factor $\lambda$ and the constant $C$ are specified in the supplement document. To ensure consistency, the concession rates of all DMs are set to zero in the original benchmarks. All remaining parameters follow the configurations reported in the corresponding original studies. Furthermore, to evaluate algorithmic robustness in scenarios without common solutions, six additional benchmark problems are constructed. For each problem, six test cases with different concession rate configurations are designed. The specific parameter configurations are detailed in the supplement document.

\subsubsection{Implementation Details}

All experiments were implemented in MATLAB R2025b and executed on a Windows 64-bit workstation with an AMD Ryzen 9 7940H processor (8 cores) and 15.2 GB RAM. The source code of the proposed evaluation framework is publicly available at: https://github.com/EvoNexusX/2026ZhaoNash.

\subsection{Experimental Results}

This subsection reports and analyzes the experimental results on both MPMOPs and MPDMPs, focusing on the comparison between the proposed Nash-product-based evaluation function ($\Psi_\mathrm{NP}$) and the conventional meanIGD metric.

\subsubsection{Results on MPMOP Benchmarks}

\begin{table}[htbp]
	\centering
	\caption{Comparison of OptMPNDS and OptAll under $\Psi_{\mathrm{NP}}$ and meanIGD for MPMOP benchmarks (D = 2). Average ranks and Friedman test results are shown in the last rows. Significance markers vs.\ OptMPNDS: Wilcoxon rank-sum\,+\,Holm correction; $++$/$--$: $p<0.01$; $+$/$-$: $p<0.05$; $\approx$: n.s.}
	\label{tab:MPMOP_results}
	\renewcommand{\arraystretch}{1.1}
	\setlength{\tabcolsep}{1.5pt}
	\footnotesize
	\begin{tabular}{@{}lccc@{}}
		\toprule
		Problem & Metric & \multicolumn{1}{c}{OptAll} & \multicolumn{1}{c}{OptMPNDS} \\
		\midrule
		\multirow{2}{*}{MPMOP1} 
		& $\Psi_{\mathrm{NP}}$ & $8.60\text{E-}10 \pm 4.36\text{E-}08\ (--)$ & $\cellcolor{lightgray} 8.34\text{E-}02 \pm 7.05\text{E-}02$ \\
		& meanIGD & $\cellcolor{lightgray} 2.98\text{E-}01 \pm 2.60\text{E-}02\ (++)$ & $7.18\text{E-}01 \pm 1.51\text{E-}01$ \\
		\midrule
		\multirow{2}{*}{MPMOP2} 
		& $\Psi_{\mathrm{NP}}$ & $2.81\text{E-}01 \pm 2.06\text{E-}01\ (--)$ & $\cellcolor{lightgray} 8.65\text{E-}01 \pm 1.32\text{E-}01$ \\
		& meanIGD & $\cellcolor{lightgray} 1.00\text{E-}02 \pm 5.10\text{E-}03\ (++)$ & $4.62\text{E-}02 \pm 9.28\text{E-}03$ \\
		\midrule
		\multirow{2}{*}{MPMOP3} 
		& $\Psi_{\mathrm{NP}}$ & $9.55\text{E-}05 \pm 3.64\text{E-}05\ (--)$ & $\cellcolor{lightgray} 5.57\text{E-}01 \pm 2.24\text{E-}01$ \\
		& meanIGD & $\cellcolor{lightgray} 1.10\text{E-}02 \pm 4.50\text{E-}03\ (++)$ & $5.49\text{E-}02 \pm 1.60\text{E-}02$ \\
		\midrule
		\multirow{2}{*}{MPMOP4} 
		& $\Psi_{\mathrm{NP}}$ & $6.42\text{E-}01 \pm 1.48\text{E-}01\ (\approx)$ & $\cellcolor{lightgray} 6.43\text{E-}01 \pm 1.58\text{E-}01$ \\
		& meanIGD & $3.28\text{E-}02 \pm 1.36\text{E-}02\ (\approx)$ & $\cellcolor{lightgray} 3.26\text{E-}02 \pm 1.35\text{E-}02$ \\
		\midrule
		\multirow{2}{*}{MPMOP5} 
		& $\Psi_{\mathrm{NP}}$ & $1.70\text{E-}01 \pm 1.15\text{E-}01\ (\approx)$ & $\cellcolor{lightgray} 1.72\text{E-}01 \pm 1.62\text{E-}01$ \\
		& meanIGD & $5.31\text{E-}02 \pm 2.61\text{E-}02\ (\approx)$ & $\cellcolor{lightgray} 5.25\text{E-}02 \pm 2.96\text{E-}02$ \\
		\midrule
		\multirow{2}{*}{MPMOP6} 
		& $\Psi_{\mathrm{NP}}$ & $5.04\text{E-}01 \pm 1.48\text{E-}01\ (\approx)$ & $\cellcolor{lightgray} 5.09\text{E-}01 \pm 1.10\text{E-}01$ \\
		& meanIGD & $6.60\text{E-}02 \pm 3.37\text{E-}02\ (\approx)$ & $\cellcolor{lightgray} 6.47\text{E-}02 \pm 3.01\text{E-}02$ \\
		\midrule
		\multirow{2}{*}{MPMOP7} 
		& $\Psi_{\mathrm{NP}}$ & $1.79\text{E-}12 \pm 2.74\text{E-}12\ (--)$ & $\cellcolor{lightgray} 2.03\text{E-}01 \pm 6.28\text{E-}01$ \\
		& meanIGD & $\cellcolor{lightgray} 6.10\text{E-}02 \pm 7.48\text{E-}02\ (++)$ & $4.24\text{E-}01 \pm 1.47\text{E-}01$ \\
		\midrule
		\multirow{2}{*}{MPMOP8} 
		& $\Psi_{\mathrm{NP}}$ & $2.87\text{E-}04 \pm 1.31\text{E-}03\ (--)$ & $\cellcolor{lightgray} 7.77\text{E-}01 \pm 1.43\text{E-}01$ \\
		& meanIGD & $\cellcolor{lightgray} 1.78\text{E-}02 \pm 1.62\text{E-}02\ (++)$ & $6.47\text{E-}02 \pm 9.61\text{E-}03$ \\
		\midrule
		\multirow{2}{*}{MPMOP9} 
		& $\Psi_{\mathrm{NP}}$ & $3.44\text{E-}02 \pm 3.31\text{E-}01\ (\approx)$ & $\cellcolor{lightgray} 3.44\text{E-}02 \pm 3.27\text{E-}01$ \\
		& meanIGD & $6.45\text{E-}02 \pm 1.49\text{E-}02\ (\approx)$ & $\cellcolor{lightgray} 6.43\text{E-}02 \pm 1.41\text{E-}02$ \\
		\midrule
		\multirow{2}{*}{MPMOP10} 
		& $\Psi_{\mathrm{NP}}$ & $\cellcolor{lightgray} 1.26\text{E-}01 \pm 1.30\text{E-}01\ (\approx)$ & $1.25\text{E-}01 \pm 1.30\text{E-}01$ \\
		& meanIGD & $\cellcolor{lightgray} 4.66\text{E-}02 \pm 2.71\text{E-}02\ (\approx)$ & $4.75\text{E-}02 \pm 2.83\text{E-}02$ \\
		\midrule
		\multirow{2}{*}{MPMOP11} 
		& $\Psi_{\mathrm{NP}}$ & $9.84\text{E-}05 \pm 6.51\text{E-}05\ (--)$ & $\cellcolor{lightgray} 3.68\text{E-}01 \pm 9.60\text{E-}02$ \\
		& meanIGD & $7.65\text{E-}02 \pm 3.36\text{E-}02\ (--)$ & $\cellcolor{lightgray} 6.55\text{E-}02 \pm 2.87\text{E-}02$ \\
		\midrule
		\multicolumn{2}{l}{Rank ($\Psi_{\mathrm{NP}}$)} & 1.727 & 1.273 \\
		\multicolumn{2}{l}{Rank (meanIGD)} & 1.455 & 1.545 \\
		\multicolumn{4}{c}{$\chi^2 (\Psi_{\mathrm{NP}})=2.2727\ (p=0.1317)$; $\chi^2 (\text{meanIGD})=0.0909\ (p=0.7630)$} \\
		\bottomrule
	\end{tabular}
\end{table}

Table~\ref{tab:MPMOP_results} shows the evaluation results of OptMPNDS and OptAll on MPMOP1--MPMOP11 under both $\Psi_\mathrm{NP}$ and meanIGD. The solution sets achieving the best $\Psi_\mathrm{NP}$ and the best meanIGD do not always coincide. For problems with well-separated individual PFs where few common solutions exist (MPMOP1-3, 7-8), OptAll approximates each PF separately, yielding low meanIGD but poor cross-party feasibility and thus suffering severe penalties. When the PFs largely overlap (MPMOP4-6, 9-10), this conflict becomes weak and both algorithms perform similarly.

\begin{figure}
	\centering
	\includegraphics[width=1\linewidth]{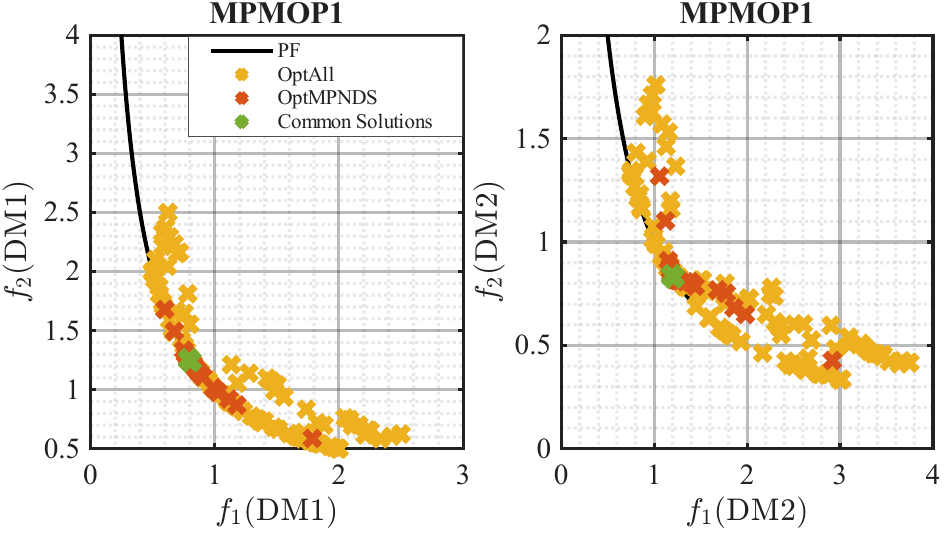}
	\caption{Representative solution sets of OptMPNDS and OptAll in the objective space of MPMOP1.}
	\label{fig:5.1.1}
\end{figure}

To illustrate these differences, Fig.~\ref{fig:5.1.1} shows representative solutions on MPMOP1. The solution set produced by OptAll (yellow points) is more evenly distributed along the PFs of individual DMs, even though it does not fully lie within the mutually acceptable region. In contrast, OptMPNDS (red points) generates solutions that are more concentrated around the common acceptable region (green points). Consequently, OptAll attains a better meanIGD score due to its emphasis on coverage of individual PFs, whereas OptMPNDS achieves a higher $\Psi_\mathrm{NP}$ by avoiding non-acceptable solutions and focusing on fairness-sensitive compromises. This discrepancy indicates meanIGD emphasizes proximity and uniformity relative to individual DMs' PFs, while $\Psi_\mathrm{NP}$ further accounts for fairness and penalizes solutions falling outside commonly acceptable regions.

\begin{figure}[htbp]
	\centering
	\includegraphics[width=1\linewidth]{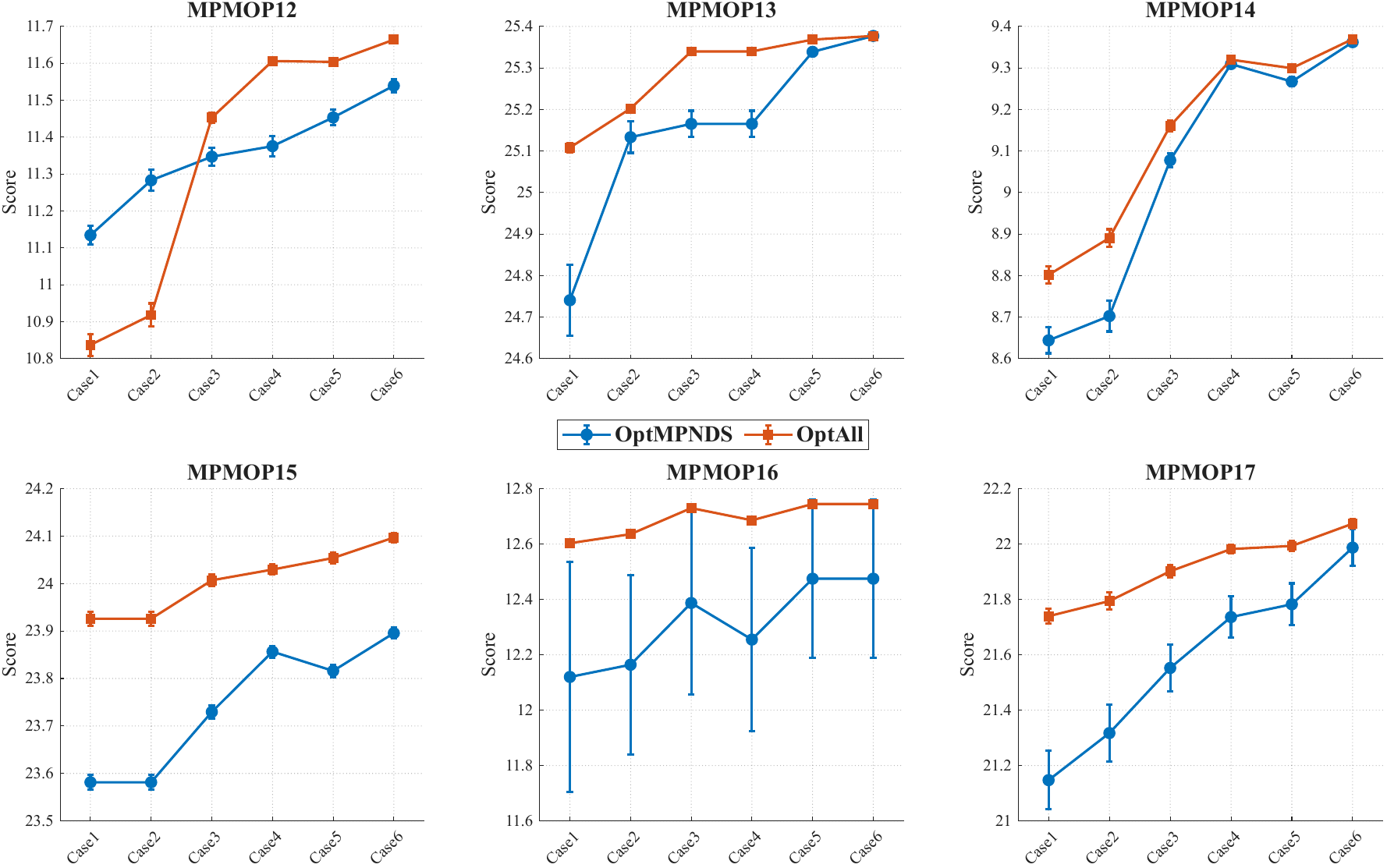}
	\caption{Evaluation scores $\Psi_{\mathrm{NP}}$ of OptMPNDS and OptAll under varying concession thresholds in no common solution scenarios of MPMOPs.}
	\label{fig:5.1.2}
\end{figure}
The evaluation is further extended to non-consensus scenarios with MPMOP12--MPMOP17, as shown in Fig.~\ref{fig:5.1.2}, where six cases of concession thresholds are considered to progressively enlarge the acceptable region. In most cases, OptAll achieves higher scores under $\Psi_\mathrm{NP}$. In the absence of common solutions, however, OptMPNDS tends to converge to the PS of one DM, resulting in asymmetric concessions, one DM incurs heavy penalties while the other receives minor penalties, leading to a lower Nash-product-based score. By contrast, OptAll, which simultaneously considers all four objectives, distributes more evenly between DMs. As a result, the product of utilities is relatively larger, yielding higher $\Psi_\mathrm{NP}$ values. As the concession threshold increases, scores of both algorithms rise. The performance gap gradually narrows, especially for MPMOP13, MPMOP14, and MPMOP17. This trend confirms that the unfairness of OptMPNDS in non-consensus problems is alleviated as the acceptable region expands, allowing previously heavily penalized solutions to contribute positively to the overall score. Fig.~\ref{fig:5.1.3} presents the PSs and solution distributions for MPMOP12 and MPMOP13. OptAll produces solutions that are relatively uniform across the objective space, while OptMPNDS generates solutions closer to the PS of DM2. This observation aligns with the trends reported in Fig.~\ref{fig:5.1.2}.

\begin{figure}[htbp]
	\centering
	\includegraphics[width=1\linewidth]{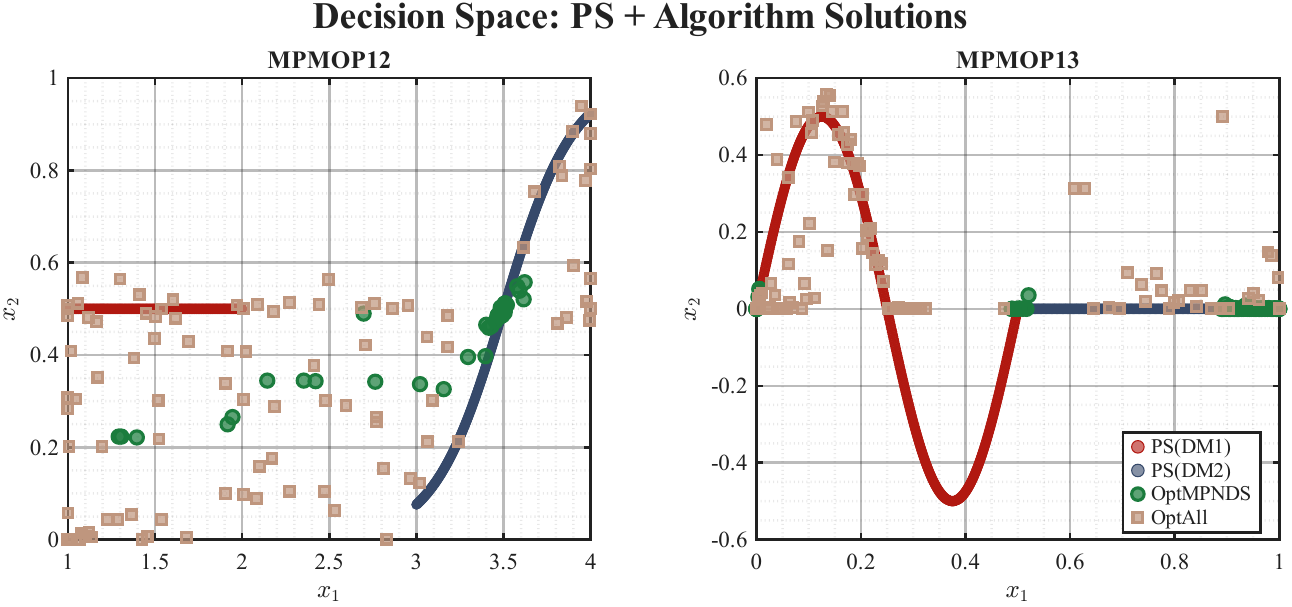}
	\caption{Representative solution distributions of OptMPNDS and OptAll in the decision space of MPMOP12 and MPMOP13.}
	\label{fig:5.1.3}
\end{figure}

For the fairness-oriented comparison, Jain's fairness index and the Gini coefficient are computed from the party-wise IGD vector $\mathbf{s}(P)=(\mathrm{IGD}_1(P),\ldots,\mathrm{IGD}_M(P))$,
where $\mathrm{IGD}_m(P)$ denotes the IGD between the final population $P$ and the PF of DM $m$. Jain's fairness index is calculated as
\begin{equation}
	J(P)=\frac{\left(\sum_{m=1}^{M}s_m(P)\right)^2}
	{M\sum_{m=1}^{M}s_m^2(P)}.
\end{equation}
The Gini coefficient is calculated as
\begin{equation}
	G(P)=\frac{\sum_{i=1}^{M}\sum_{j=1}^{M}|s_i(P)-s_j(P)|}
	{2M\sum_{m=1}^{M}s_m(P)}.
	\end{equation}

\begin{table}[!t]
	\centering
	\caption{Fairness-oriented comparison of the final solution sets in MPMOP12 and MPMOP13. Larger Jain's fairness index and smaller Gini coefficient indicate a more balanced distribution of party-wise IGD values.}
	\label{tab:jain_gini}
	\begin{tabular}{@{}l l c c@{}}
			\toprule
			\textbf{Problem} & \textbf{Algorithm} & \textbf{Jain} $\uparrow$ & \textbf{Gini} $\downarrow$ \\
			\midrule
			\multirow[c]{2}{*}{MPMOP12}
			& OptMPNDS & 0.8440 & 0.2150 \\
			& OptAll   & \cellcolor{lightgray}\textbf{0.9350} & \cellcolor{lightgray}\textbf{0.1318} \\
			\midrule
			\multirow[c]{2}{*}{MPMOP13}
			& OptMPNDS & 0.7231 & 0.3094 \\
			& OptAll   & \cellcolor{lightgray}\textbf{0.7722} & \cellcolor{lightgray}\textbf{0.2716} \\
			\bottomrule
	\end{tabular}
\end{table}

Table~\ref{tab:jain_gini} further reports Jain's fairness index and the Gini coefficient computed from the party-wise performance values of the final populations. On MPMOP12, OptAll achieves a substantially larger Jain index and a smaller Gini coefficient, indicating a more balanced distribution between the two DMs. A similar trend is observed on MPMOP13. These results are consistent with Fig.~5 and the proposed $\Psi_{\mathrm{NP}}$, confirming that OptMPNDS tends to concentrate near the PS of one DM, whereas OptAll yields a more balanced compromise in these non-consensus cases.

\subsubsection{Results on MPDMP Benchmarks}

Table~\ref{tab:MPDMP_results} summarizes the performance of different algorithms on the original MPDMP benchmarks using both meanIGD and $\Psi_{\mathrm{NP}}$. Since all algorithms are able to achieve relatively good convergence on these problems, the resulting meanIGD values are generally small. Consequently, the corresponding values of $\Psi_{\mathrm{NP}}$ become very large, and thus the logarithmic scale is adopted for numerical stability and clearer comparison.

A pronounced ranking discrepancy is observed between meanIGD and $\log\Psi_{\mathrm{NP}}$. Under meanIGD, OptAll achieves the best result on all twelve MPDMP instances and obtains the best average rank of $1.000$. In contrast, under $\log\Psi_{\mathrm{NP}}$, OptAll is ranked last overall, with an average rank of $3.750$, whereas OptMPNDS3 achieves the best average rank of $1.417$. The Friedman tests further confirm significant performance differences among the compared algorithms under both metrics. This systematic ranking reversal reflects the different solution qualities emphasized by the two indicators. MeanIGD rewards geometric proximity to the reference sets after aggregating the party-wise approximation errors. Consequently, OptAll can achieve small meanIGD values by generating solutions that are close to the individual Pareto sets in an average sense. However, such a solution set may still contain solutions that are substantially less acceptable to one DM than to the others. In contrast, $\Psi_{\mathrm{NP}}$ evaluates the party-wise losses jointly and penalizes solution sets with unbalanced or mutually unacceptable compromises. Therefore, the substantially better $\log\Psi_{\mathrm{NP}}$ values obtained by the multiparty optimizers indicate that their solution sets provide more balanced compromises among DMs, even though their meanIGD values are larger.

\begin{table*}[htbp]
	\centering
	\caption{Comparison of OptMPNDS variants and OptAll under $\log \Psi_{\mathrm{NP}}$ and meanIGD for MPDMP benchmarks (mean\,$\pm$\,std over 30 runs). Significance markers vs.\ OptMPNDS3: Wilcoxon rank-sum\,+\,Holm correction; $++$/$--$: $p<0.01$; $+$/$-$: $p<0.05$; $\approx$: n.s.}
	\label{tab:MPDMP_results}
	\begin{tabular}{l l >{\centering\arraybackslash}p{0.18\linewidth} >{\centering\arraybackslash}p{0.18\linewidth} >{\centering\arraybackslash}p{0.18\linewidth} >{\centering\arraybackslash}p{0.18\linewidth}}
		\toprule
		Problem & Metric & OptAll & OptMPNDS & OptMPNDS2 & OptMPNDS3 \\
		\midrule
		\multirow{2}{*}{MPDMP1} & $\log \Psi_{\mathrm{NP}}$ & $5.3668 \pm 0.0007(--)$ & $5.3721 \pm 0.0028(\approx)$ & $5.3726 \pm 0.0019(\approx)$ & $\cellcolor{lightgray} 5.3727 \pm 0.0026$ \\
		& meanIGD & $\cellcolor{lightgray} 9.04\text{E-}01 \pm 7.20\text{E-}02(++)$ & $1.39\text{E+}01 \pm 1.60\text{E+}00(\approx)$ & $1.36\text{E+}01 \pm 1.10\text{E+}00(\approx)$ & $1.36\text{E+}01 \pm 1.50\text{E+}00$ \\
		\midrule
		\multirow{2}{*}{MPDMP2} & $\log \Psi_{\mathrm{NP}}$ & $5.3744 \pm 0.0008(--)$ & $5.3782 \pm 0.0034(\approx)$ & $\cellcolor{lightgray} 5.3790 \pm 0.0039(\approx)$ & $5.3777 \pm 0.0033$ \\
		& meanIGD & $\cellcolor{lightgray} 7.22\text{E-}01 \pm 7.60\text{E-}02(++)$ & $1.05\text{E+}01 \pm 2.00\text{E+}00(\approx)$ & $ 1.01\text{E+}01 \pm 2.30\text{E+}00(\approx)$ & $1.08\text{E+}01 \pm 1.90\text{E+}00$ \\
		\midrule
		\multirow{2}{*}{MPDMP3} & $\log \Psi_{\mathrm{NP}}$ & $5.3728 \pm 0.0007(--)$ & $5.3807 \pm 0.0008(--)$ & $5.3808 \pm 0.0008(--)$ & $\cellcolor{lightgray} 5.3814 \pm 0.0008$ \\
		& meanIGD & $\cellcolor{lightgray} 1.18\text{E+}00 \pm 6.00\text{E-}02(++)$ & $7.73\text{E+}00 \pm 4.70\text{E-}01(--)$ & $7.64\text{E+}00 \pm 4.70\text{E-}01(--)$ & $ 7.28\text{E+}00 \pm 4.90\text{E-}01$ \\
		\midrule
		\multirow{2}{*}{MPDMP4} & $\log \Psi_{\mathrm{NP}}$ & $5.3460 \pm 0.0012(--)$ & $5.3527 \pm 0.0028(--)$ & $5.3533 \pm 0.0028(--)$ & $\cellcolor{lightgray} 5.3591 \pm 0.0021$ \\
		& meanIGD & $\cellcolor{lightgray} 2.39\text{E+}00 \pm 1.50\text{E-}01(++)$ & $1.93\text{E+}01 \pm 1.50\text{E+}00(--)$ & $1.89\text{E+}01 \pm 1.50\text{E+}00(--)$ & $ 1.55\text{E+}01 \pm 1.20\text{E+}00$ \\
		\midrule
		\multirow{2}{*}{MPDMP5} & $\log \Psi_{\mathrm{NP}}$ & $5.3607 \pm 0.0010(--)$ & $5.3771 \pm 0.0005(\approx)$ & $5.3771 \pm 0.0004(\approx)$ & $\cellcolor{lightgray} 5.3773 \pm 0.0004$ \\
		& meanIGD & $\cellcolor{lightgray} 3.50\text{E+}00 \pm 1.50\text{E-}01(++)$ & $8.67\text{E+}00 \pm 3.10\text{E-}01(\approx)$ & $8.67\text{E+}00 \pm 2.40\text{E-}01(\approx)$ & $ 8.63\text{E+}00 \pm 2.70\text{E-}01$ \\
		\midrule
		\multirow{2}{*}{MPDMP6} & $\log \Psi_{\mathrm{NP}}$ & $5.3653 \pm 0.0007(--)$ & $ 5.3683 \pm 0.0002(++)$ &\cellcolor{lightgray} $5.3683 \pm 0.0002(++)$ & $5.3678 \pm 0.0000$ \\
		& meanIGD & $\cellcolor{lightgray} 2.06\text{E+}00 \pm 1.00\text{E-}01(++)$ & $1.43\text{E+}01 \pm 1.00\text{E-}02(++)$ & $1.43\text{E+}01 \pm 1.00\text{E-}02(++)$ & $1.46\text{E+}01 \pm 0.00\text{E+}00$ \\
		\midrule
		\multirow{2}{*}{MPDMP7} & $\log \Psi_{\mathrm{NP}}$ & $5.3553 \pm 0.0012(--)$ & $5.3642 \pm 0.0003(\approx)$ & $\cellcolor{lightgray} 5.3643 \pm 0.0004(\approx)$ & $5.3643 \pm 0.0004$ \\
		& meanIGD & $\cellcolor{lightgray} 3.79\text{E+}00 \pm 2.00\text{E-}01(++)$ & $1.60\text{E+}01 \pm 2.00\text{E-}02(\approx)$ & $ 1.60\text{E+}01 \pm 2.00\text{E-}02(\approx)$ & $1.61\text{E+}01 \pm 2.00\text{E-}02$ \\
		\midrule
		\multirow{2}{*}{MPDMP8} & $\log \Psi_{\mathrm{NP}}$ & $5.3647 \pm 0.0007(--)$ & $5.3760 \pm 0.0002(\approx)$ & $5.3760 \pm 0.0003(\approx)$ & $\cellcolor{lightgray} 5.3761 \pm 0.0003$ \\
		& meanIGD & $\cellcolor{lightgray} 2.98\text{E+}00 \pm 1.20\text{E-}01(++)$ & $9.44\text{E+}00 \pm 1.40\text{E-}01(\approx)$ & $ 9.42\text{E+}00 \pm 1.50\text{E-}01(\approx)$ & $9.47\text{E+}00 \pm 1.60\text{E-}01$ \\
		\midrule
		\multirow{2}{*}{MPDMP9} & $\log \Psi_{\mathrm{NP}}$ & $8.0609 \pm 0.0009(--)$ & $8.0653 \pm 0.0030(--)$ & $8.0662 \pm 0.0042(--)$ & $\cellcolor{lightgray} 8.0706 \pm 0.0032$ \\
		& meanIGD & $\cellcolor{lightgray} 9.68\text{E-}01 \pm 9.20\text{E-}02(++)$ & $1.15\text{E+}01 \pm 1.10\text{E+}00(--)$ & $1.11\text{E+}01 \pm 1.60\text{E+}00(--)$ & $9.34\text{E+}00 \pm 1.22\text{E+}00$ \\
		\midrule
		\multirow{2}{*}{MPDMP10} & $\log \Psi_{\mathrm{NP}}$ & $\cellcolor{lightgray} 8.0693 \pm 0.0008(++)$ & $8.0617 \pm 0.0031(\approx)$ & $8.0610 \pm 0.0031(\approx)$ & $ 8.0624 \pm 0.0025$ \\
		& meanIGD & $ \cellcolor{lightgray} 1.52\text{E+}00 \pm 8.00\text{E-}02(++)$ & $1.27\text{E+}01 \pm 1.20\text{E+}00(\approx)$ & $1.29\text{E+}01 \pm 1.10\text{E+}00(\approx)$ & $ 1.24\text{E+}01 \pm 9.00\text{E-}01$ \\
		\midrule
		\multirow{2}{*}{MPDMP11} & $\log \Psi_{\mathrm{NP}}$ & $8.0429 \pm 0.0015(--)$ & $8.0689 \pm 0.0011(\approx)$ & $8.0690 \pm 0.0008(-)$ & $\cellcolor{lightgray} 8.0695 \pm 0.0012$ \\
		& meanIGD & $\cellcolor{lightgray} 1.94\text{E+}00 \pm 1.10\text{E-}01(++)$ & $7.78\text{E+}00 \pm 4.70\text{E-}01(\approx)$ & $7.79\text{E+}00 \pm 3.60\text{E-}01(\approx)$ & $ 7.75\text{E+}00 \pm 4.80\text{E-}01$ \\
		\midrule
		\multirow{2}{*}{MPDMP12} & $\log \Psi_{\mathrm{NP}}$ & $8.0453 \pm 0.0011(--)$ & $8.0661 \pm 0.0004(--)$ & $8.0661 \pm 0.0004(--)$ & $\cellcolor{lightgray} 8.0666 \pm 0.0006$ \\
		& meanIGD & $\cellcolor{lightgray} 2.92\text{E+}00 \pm 1.80\text{E-}01(++)$ & $9.15\text{E+}00 \pm 1.60\text{E-}01(--)$ & $9.16\text{E+}00 \pm 1.50\text{E-}01(--)$ & $ 9.00\text{E+}00 \pm 2.30\text{E-}01$ \\
		\bottomrule
		\multicolumn{2}{l}{Rank ($\Psi_{\mathrm{NP}}$)} & 3.750 & 2.667 & 2.167 & 1.417\\
		\multicolumn{2}{l}{Rank (meanIGD)} & 1.000 & 3.250 & 3.083 & 2.667 \\
		\multicolumn{6}{l}{$\chi^2 (\Psi_{\mathrm{NP}})=20.70\ (p=0.000122)$; $\chi^2 (\text{meanIGD})=22.9000\ (p=0.000042)$; CD = 1.3540} \\
		\bottomrule
	\end{tabular}
\end{table*}

\begin{figure}
	\centering
	\includegraphics[width=1\linewidth]{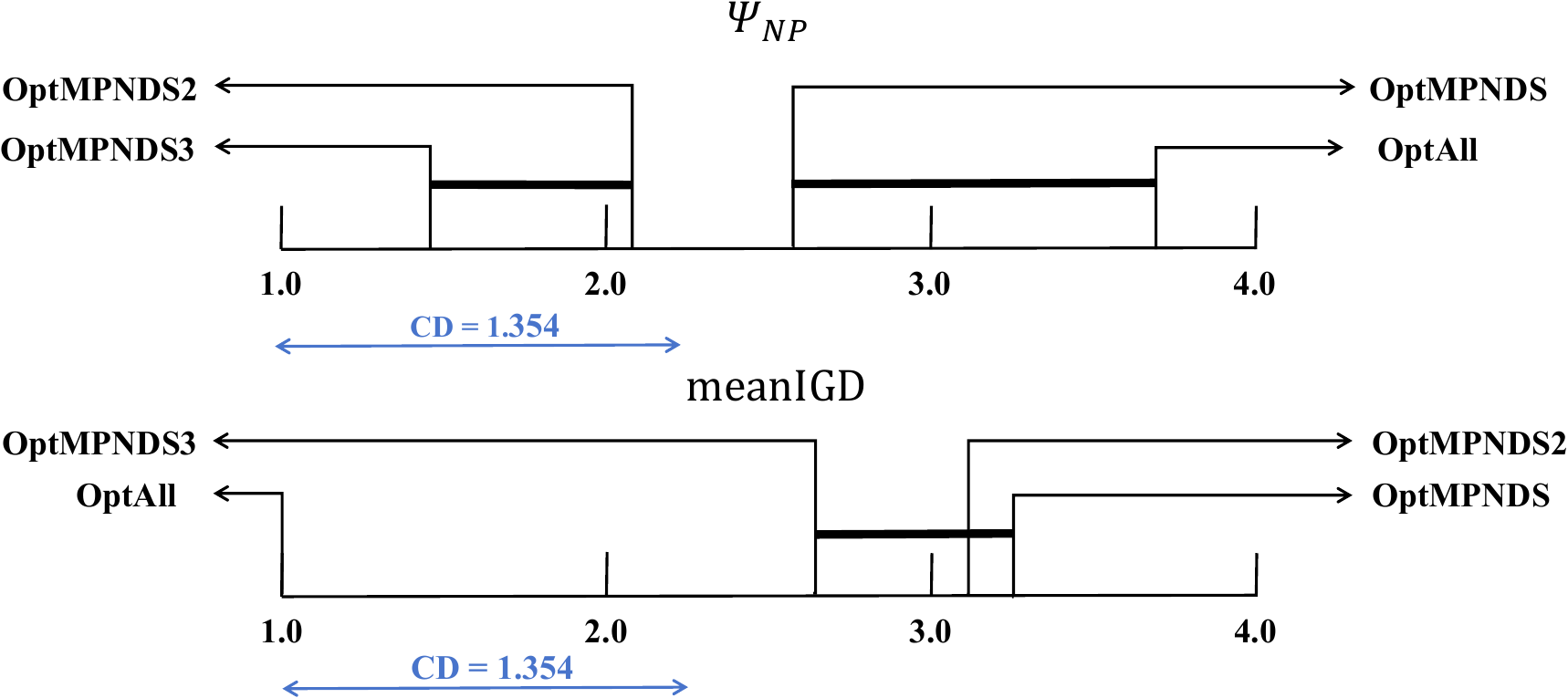}
	\caption{Critical difference plots of $\Psi_{\mathrm{NP}}$ and meanIGD under MPDMP1-12.}
	\label{fig:5.2.3}
\end{figure}

As shown in Fig.~\ref{fig:5.2.3}, the rankings under $\log\Psi_{\mathrm{NP}}$ and meanIGD are inconsistent: OptMPNDS3 ranks first under $\log\Psi_{\mathrm{NP}}$, whereas OptAll performs best under meanIGD. Because OptAll is designed to minimise the IGD of concatenated objective space, it naturally excels on this purely geometric metric. However, it does not penalize excessive sacrifices of one DM's objectives for another DM's benefit. In contrast, $\Psi_{\mathrm{NP}}$ incorporates the concession penalty $L_m^{\mathrm{pen}}$, and therefore favors more balanced and consensus-oriented solution sets.

\begin{figure}[htbp]
	\centering
	\includegraphics[width=1\linewidth]{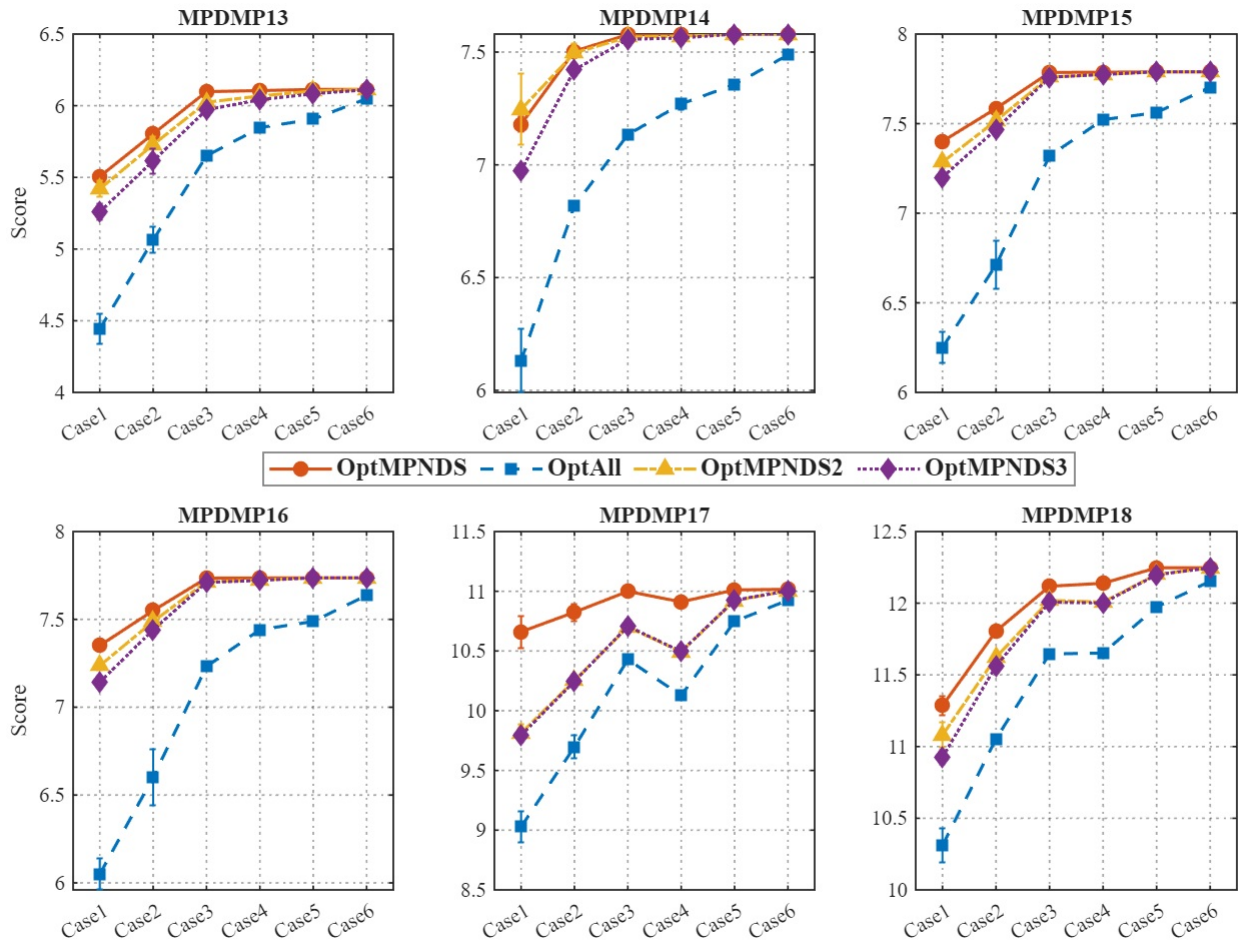}
	\caption{Evaluation scores $\log\Psi_{\mathrm{NP}}$ under varying concession thresholds in non common solution scenarios of MPDMPs.}
	\label{fig:5.2.1}
\end{figure}

\begin{figure}[htbp]
	\centering
	\includegraphics[width=0.77\linewidth]{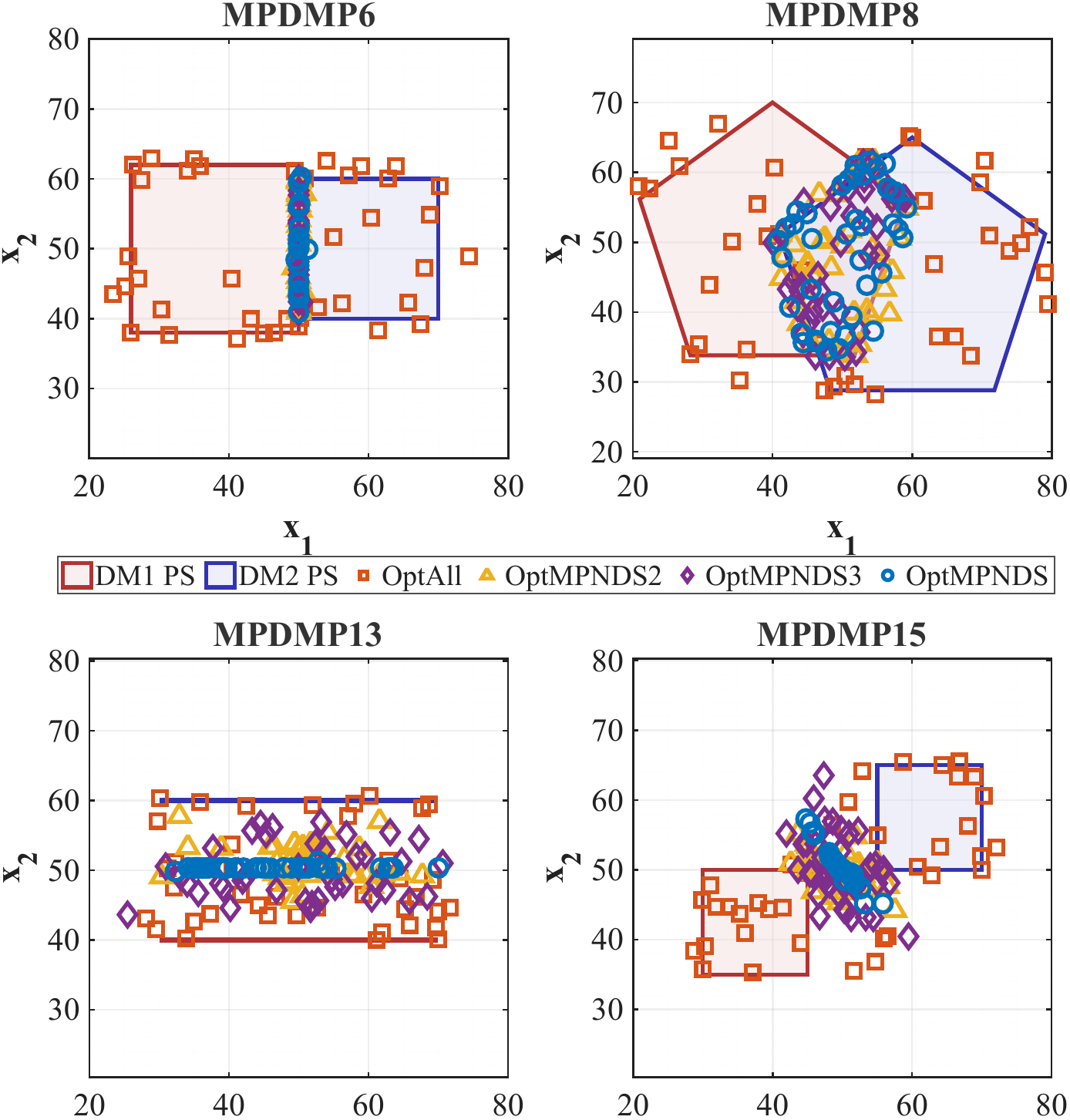}
	\caption{Representative solution distributions of algorithms on decision space of MPDMP6, 8, 13, and 15.}
	\label{fig:5.2.2}
\end{figure}

Fig.~\ref{fig:5.2.1} further shows the evaluation results on the constructed MPDMP instances with varying sizes of the common acceptable region $\mathcal{A}$. OptMPNDS achieves the highest evaluation scores across all cases, while OptAll remains the worst-performing algorithm. Moreover, for all algorithms, the evaluation score $\Psi_{\mathrm{NP}}$ increases monotonically as the acceptable region expands, reflecting the reduced difficulty in satisfying mutual acceptability. The performance gap among algorithms is most pronounced in Case~1, where the acceptable region is empty, whereas in Case~6 the scores of different algorithms become much closer, indicating diminishing fairness discrimination when concessions are sufficiently large.

More detailed insights can be obtained from Fig.~\ref{fig:5.2.2}, which visualizes representative decision-space distributions for problems with and without common acceptable solutions. For MPDMP6 and MPDMP8, where common acceptable regions exist, all three variants of OptMPNDS converge tightly toward the shared region, leading to very similar $\Psi_{\mathrm{NP}}$ values. In contrast, OptAll produces solutions that are more dispersed and significantly farther from the common region.

For MPDMP13 and MPDMP15, where no common solution exists, OptMPNDS still exhibits a distinctive behavior. Its solutions concentrate along a central compromise manifold between the PSs of the two DMs, indicating a tendency toward balanced trade-offs. In contrast, the remaining algorithms generate solutions scattered around this region, with OptAll showing the widest dispersion. As a result, OptMPNDS attains the highest evaluation scores, while OptAll is subject to substantially stronger penalties under the concession-aware evaluation.

Overall, two observations can be made from these results. First, the proposed evaluation framework is capable of distinguishing algorithms not only by convergence quality but also by their ability to achieve fairness and mutual acceptability. Second, these results suggest that algorithm design for multi-party decision-making should not be limited to pursuing consensus solutions alone. The degree of concession between parties, which effectively acts as a controllable parameter in practical negotiations, must also be explicitly incorporated into both algorithmic mechanisms and performance evaluation.

\section{Discussion}
\label{sec:discussion}

\subsection{Sensitivity of $\Psi_{\mathrm{NP}}$ to Concession Levels}

This subsection examines how the proposed Nash-product-based evaluation function $\Psi_{\mathrm{NP}}$ responds to variations in concession, with particular attention to the interaction between the intrinsic concession level of a solution population and the concession threshold prescribed in the evaluation metric. The experiments aim to demonstrate whether the observed score landscapes align with the design principles of $\Psi_{\mathrm{NP}}$.

\begin{figure}
	\centering
	\includegraphics[width=1\linewidth]{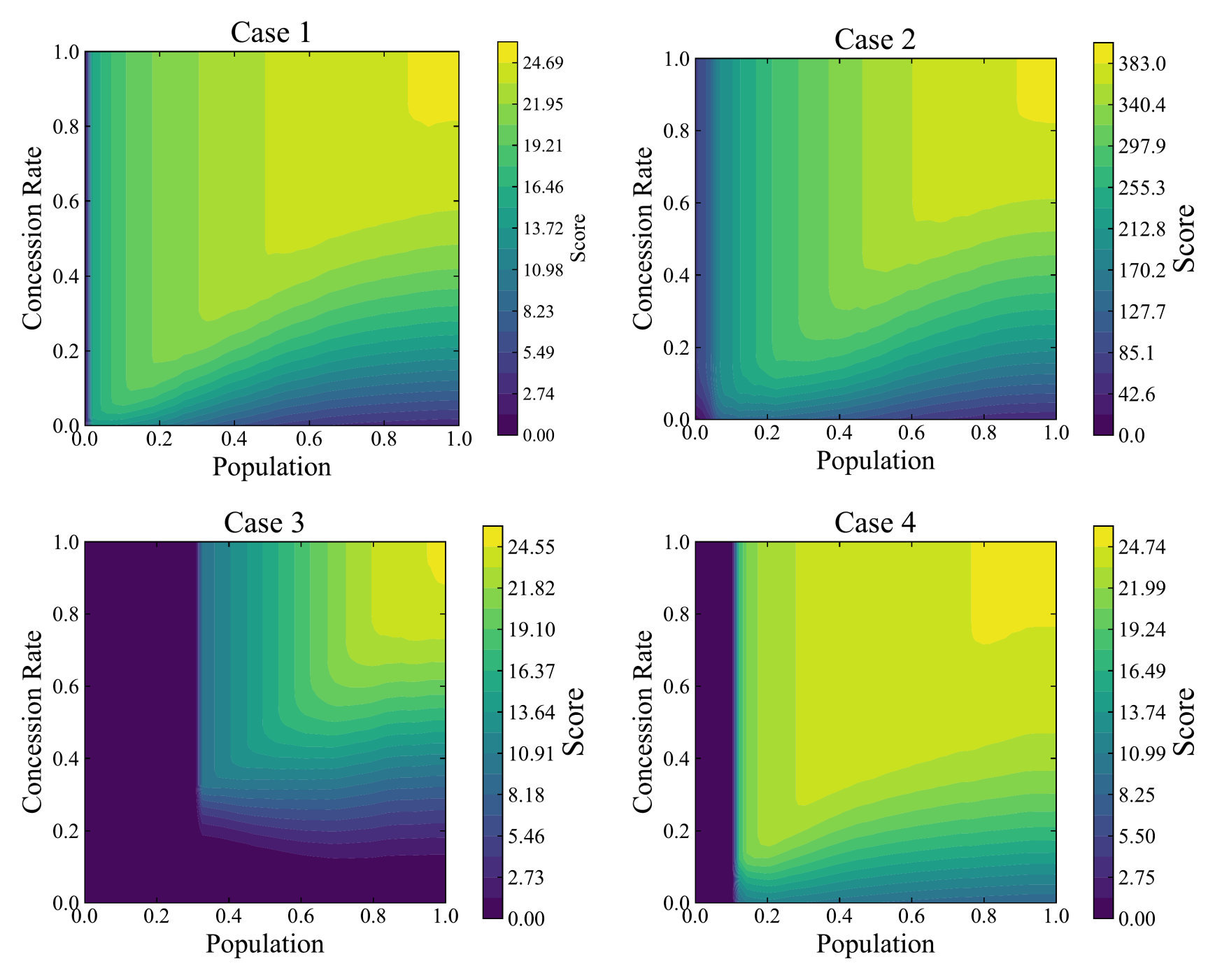}
	\caption{Contour plots of $\Psi_{\mathrm{NP}}$ under symmetric concession thresholds and population dispersion.}
	\label{fig:4.3.1}
\end{figure}

\emph{(1) Symmetric concession thresholds.}The first experiment considers four multi-party distance minimization problems. The first two problems admit inherent common solutions, whereas the latter two do not. For each problem, a solution population is generated within the acceptable region $\mathcal{A}$ with a prescribed intrinsic concession level $\gamma$, while the evaluation metric adopts an identical concession threshold $\hat{\gamma}$ for all DMs. By varying $(\gamma, \hat{\gamma})$, contour plots of $\Psi_{\mathrm{NP}}$ are obtained, as shown in Fig.~\ref{fig:4.3.1}.

The contour plots can be interpreted from two complementary perspectives. First, fixing a population generated in $\mathcal{A}$ with concession level $\gamma$, if the evaluation parameter satisfies $\hat{\gamma} \ge \gamma$, the resulting $\Psi_{\mathrm{NP}}$ value remains unchanged. In this case, the population fully satisfies the prescribed concession requirement and no penalty is incurred. In contrast, when $\hat{\gamma} < \gamma$, the evaluation score decreases monotonically as $\hat{\gamma}$ decreases, since part of the population is regarded as insufficiently acceptable and is therefore penalized.

Second, fixing the concession threshold $\hat{\gamma}$ in the evaluation metric and varying the intrinsic concession level $\gamma$ of the solution population, two trends are observed. When $\hat{\gamma}$ is relatively small, the evaluation score initially increases with $\gamma$ and then decreases. At small $\gamma$, the population is highly concentrated and its coverage of $\mathcal{A}$ is limited, resulting in a low score. As $\gamma$ increases and approaches $\hat{\gamma}$, the penalty vanishes and the score reaches its maximum. Further increasing $\gamma$ enlarges the population range, but since part of the population exceeds the prescribed threshold, penalties are reintroduced and the score decreases. In contrast, when $\hat{\gamma}$ is sufficiently large, the score increases monotonically with $\gamma$, because the population expansion improves coverage without penalties.

For the two problems without inherent common solutions, the same qualitative trends are observed. However, an extended zero-score region appears at small $\gamma$, reflecting the absence of any common acceptable region $\mathcal{A}$. Only after sufficient concession enlarges $\mathcal{A}$ does $\Psi_{\mathrm{NP}}$ become positive. This behavior is consistent with the design of the metric, which assigns no value when mutual acceptability is not achieved.

\begin{figure}
	\centering
	\includegraphics[width=0.95\linewidth]{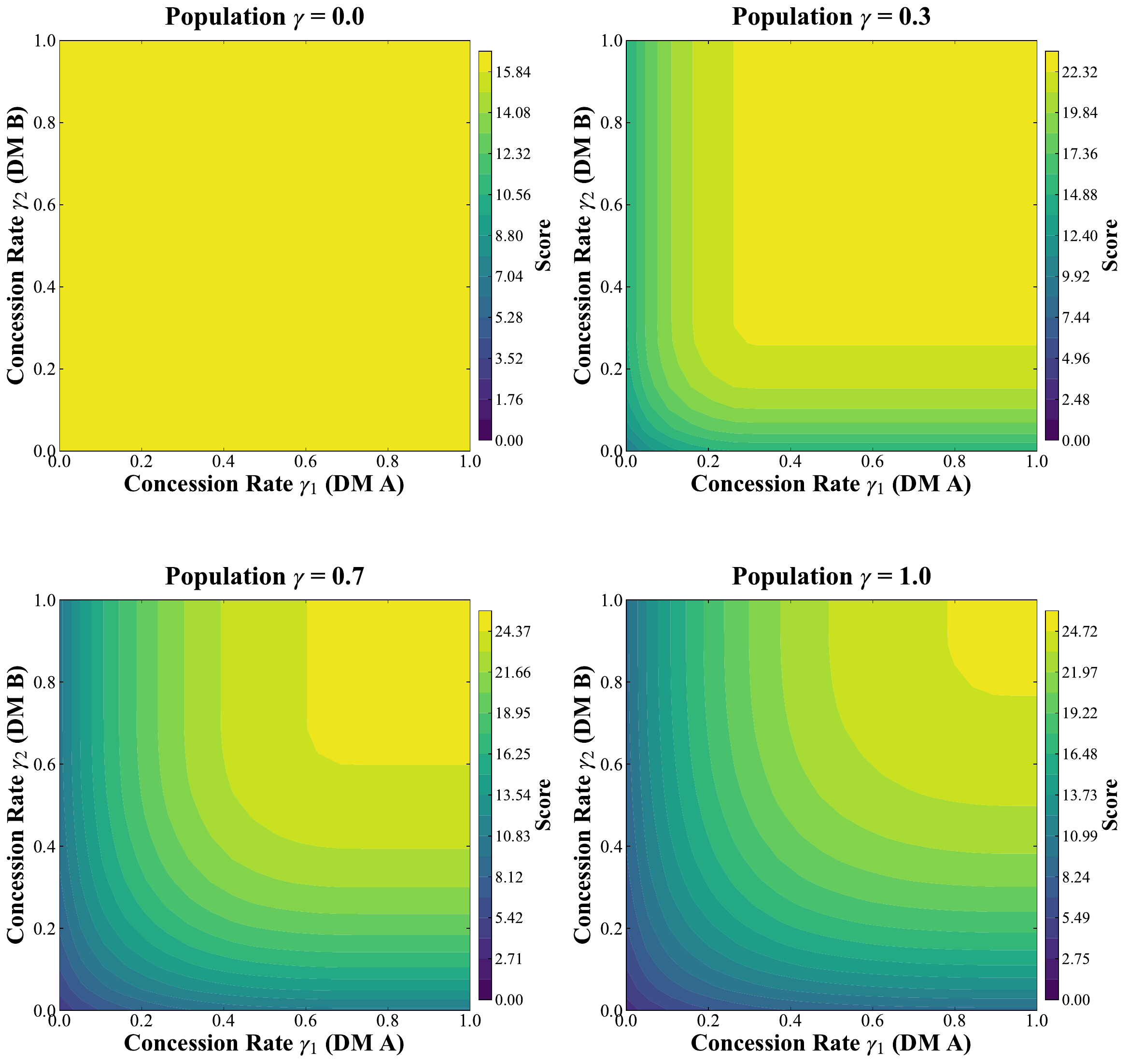}
	\caption{Contour plots of $\Psi_{\mathrm{NP}}$ under asymmetric concession thresholds.}
	\label{fig:4.3.2}
\end{figure}

\emph{(2) Asymmetric concession thresholds. }The second experiment investigates asymmetric concession thresholds in a multi-party distance minimization problem with inherent common solutions. Solution populations are generated under four intrinsic concession levels ($\gamma=0, 0.3, 0.7$, and $1.0$), while DMs adopt different concession thresholds. The resulting contour plots are shown in Fig.~\ref{fig:4.3.2}.

For a fixed solution population, relaxing the concession threshold of one DM leads to a monotonic increase in $\Psi_{\mathrm{NP}}$, indicating reduced penalties. Meanwhile, populations with larger intrinsic concession levels require correspondingly larger thresholds to avoid score degradation. The resulting score contours form a convex frontier, reflecting the trade-off induced by asymmetric concession requirements among DMs. This convexity requirement is associated with solving the classical Nash bargaining problem. In this paper, however, the Nash product is used only as an evaluation functional for attainable MPMOP solution sets, rather than as an optimization rule over a convex utility set.

Overall, the two experiments illustrate that $\Psi_{\mathrm{NP}}$ produces structured and interpretable score landscapes in response to both population-level concession and evaluation-level concession thresholds. The observed trends are consistent across problems with and without inherent common solutions, and align with the intended penalty and acceptability mechanisms embedded in the metric. While these experiments do not constitute a definitive validation, they provide empirical evidence that the proposed evaluation function behaves in a manner consistent with its conceptual design.

\subsection{PF-Free Evaluation: A UAV Path Planning Case Study}

One practical concern of the proposed Nash-product-based evaluation framework is its applicability when true PFs are unavailable. In such scenarios, conventional PF-dependent penalties cannot be defined, and the evaluation necessarily focuses on fairness and relative balance rather than absolute performance. To examine whether the proposed framework remains usable under this constraint, we consider the multi-party multi-objective UAV path planning benchmark proposed in~\cite{chen2024evolutionary}, where explicit PFs cannot be obtained.

Table~\ref{tab:UAV_path_planning_results_summary} reports the evaluation results using three PF-free metrics: the sum of hypervolumes across DMs (sumHV), the product of absolute hypervolume values (Absolute), and a comparative Nash-product-based metric derived from normalized utilities (Comparative). This experiment demonstrates that the proposed Nash-based formulation can be instantiated in PF-free settings, providing a practical evaluation option when true PFs are unavailable.

Let $G(P)$ denote a common strictly monotone gain-type performance metric used to evaluate solution sets when PFs are unknown. For each DM $m$, a comparative utility is defined as
\[
u_m(P)=\frac{G_m(P)}{\max_{P'\in\mathcal{P}} G_m(P')} \in (0,1],
\]
where $\mathcal{P}$ denotes the collection of solution sets generated by all compared algorithms. This normalization preserves ordinal preferences within each DM while mapping heterogeneous performance scales to a common unit interval. Based on these comparative utilities, a PF-free Nash-type score is defined as
\[
\Psi_{\mathrm{NP}} = \prod_{m=1}^{M} u_m(P),
\]
which evaluates the relative balance of performance across DMs without relying on explicit PF information or concession-based penalties.

Overall, the resulting algorithm rankings are largely consistent: algorithms that perform well under sumHV also tend to obtain high scores under the Absolute and Comparative metrics. The comparative Nash formulation yields algorithm rankings largely consistent with sumHV on these MPUAV instances, while additionally accounting for inter-party balance. This suggests the PF-free instantiation produces reasonable evaluations when PF information is unavailable.

\begin{table*}[htbp]
	\centering
	\caption{Comparison of multi-party multi-objective optimization algorithms under $\Psi_{\mathrm{NP}}$ and sumHV for MPUAV path planning benchmarks.}
	\label{tab:UAV_path_planning_results_summary}
	
	\resizebox{\textwidth}{!}{
		\small
		\setlength{\tabcolsep}{3.5pt}
		\begin{tabular}{lcccccc}
			\toprule
			Problem & metric & OptAll & OptMPNDS & OptMPNDS2 & MPNNIA & MPAIMA\\
			\midrule
			
			\multirow{3}{*}{MPUAV1} 
			& Comparative & $3.50\text{E-01}\,(1.10\text{E-01})$ & $5.70\text{E-01}\,(9.47\text{E-02})$ & $6.20\text{E-01}\,(1.07\text{E-01})$ & $6.11\text{E-01}\,(1.06\text{E-01})$ & $\cellcolor{lightgray} 7.41\text{E-01}\,(1.23\text{E-01})$ \\
			& Absolute& $1.22\text{E+04}\,(3.84\text{E+03})$ & $1.98\text{E+04}\,(3.29\text{E+03})$ & $2.16\text{E+04}\,(3.73\text{E+03})$ & $2.13\text{E+04}\,(3.68\text{E+03})$ & $\cellcolor{lightgray} 2.58\text{E+04}\,(4.28\text{E+03})$ \\
			& sumHV & $9.51\text{E+05}\,(2.16\text{E+05})$ & $1.23\text{E+06}\,(1.69\text{E+05})$ & $1.25\text{E+06}\,(1.80\text{E+05})$ & $1.23\text{E+06}\,(1.77\text{E+05})$ & $\cellcolor{lightgray} 1.42\text{E+06}\,(1.63\text{E+05})$ \\
			\midrule
			
			\multirow{3}{*}{MPUAV2} 
			& Comparative & $3.51\text{E-01}\,(1.05\text{E-01})$ & $6.19\text{E-01}\,(1.09\text{E-01})$ & $6.44\text{E-01}\,(1.37\text{E-01})$ & $6.51\text{E-01}\,(1.51\text{E-01})$ & $\cellcolor{lightgray} 7.54\text{E-01}\,(1.27\text{E-01})$ \\
			& Absolute & $2.65\text{E+04}\,(7.90\text{E+03})$ & $4.66\text{E+04}\,(8.25\text{E+03})$ & $4.86\text{E+04}\,(1.03\text{E+04})$ & $4.91\text{E+04}\,(1.14\text{E+04})$ & $\cellcolor{lightgray} 5.69\text{E+04}\,(9.58\text{E+03})$ \\
			& sumHV & $1.96\text{E+06}\,(4.23\text{E+05})$ & $2.94\text{E+06}\,(3.43\text{E+05})$ & $2.89\text{E+06}\,(4.63\text{E+05})$ & $2.86\text{E+06}\,(5.27\text{E+05})$ & $\cellcolor{lightgray} 3.10\text{E+06}\,(3.61\text{E+05})$ \\
			\midrule
			
			\multirow{3}{*}{MPUAV3} 
			& Comparative & $3.69\text{E-01}\,(1.26\text{E-01})$ & $6.25\text{E-01}\,(1.27\text{E-01})$ & $6.17\text{E-01}\,(1.31\text{E-01})$ & $6.84\text{E-01}\,(1.03\text{E-01})$ & $\cellcolor{lightgray} 7.51\text{E-01}\,(1.08\text{E-01})$ \\
			& Absolute & $4.46\text{E+04}\,(1.52\text{E+04})$ & $7.56\text{E+04}\,(1.53\text{E+04})$ & $7.46\text{E+04}\,(1.59\text{E+04})$ & $8.27\text{E+04}\,(1.24\text{E+04})$ & $\cellcolor{lightgray} 9.08\text{E+04}\,(1.31\text{E+04})$ \\
			& sumHV & $2.86\text{E+06}\,(7.40\text{E+05})$ & $4.04\text{E+06}\,(5.99\text{E+05})$ & $3.84\text{E+06}\,(6.75\text{E+05})$ & $4.10\text{E+06}\,(5.34\text{E+05})$ & $\cellcolor{lightgray} 4.35\text{E+06}\,(4.87\text{E+05})$ \\
			\midrule
			
			\multirow{3}{*}{MPUAV4} 
			& Comparative & $3.98\text{E-01}\,(1.37\text{E-01})$ & $6.35\text{E-01}\,(1.00\text{E-01})$ & $6.00\text{E-01}\,(1.27\text{E-01})$ & $5.68\text{E-01}\,(1.26\text{E-01})$ & $\cellcolor{lightgray} 6.83\text{E-01}\,(1.19\text{E-01})$ \\
			& Absolute& $1.50\text{E+04}\,(5.15\text{E+03})$ & $2.39\text{E+04}\,(3.77\text{E+03})$ & $2.26\text{E+04}\,(4.79\text{E+03})$ & $2.14\text{E+04}\,(4.73\text{E+03})$ & $\cellcolor{lightgray} 2.57\text{E+04}\,(4.49\text{E+03})$ \\
			& sumHV & $1.06\text{E+06}\,(2.35\text{E+05})$ & $1.20\text{E+06}\,(1.57\text{E+05})$ & $1.12\text{E+06}\,(2.32\text{E+05})$ & $1.03\text{E+06}\,(2.42\text{E+05})$ & $\cellcolor{lightgray} 1.31\text{E+06}\,(1.98\text{E+05})$ \\
			\midrule
			
			\multirow{3}{*}{MPUAV5} 
			& Comparative & $3.47\text{E-01}\,(1.16\text{E-01})$ & $6.23\text{E-01}\,(1.19\text{E-01})$ & $\cellcolor{lightgray} 6.53\text{E-01}\,(1.48\text{E-01})$ & $5.84\text{E-01}\,(1.28\text{E-01})$ & $6.13\text{E-01}\,(1.27\text{E-01})$ \\
			& Absolute  & $2.75\text{E+04}\,(9.18\text{E+03})$ & $4.93\text{E+04}\,(9.40\text{E+03})$ & $\cellcolor{lightgray} 5.17\text{E+04}\,(1.17\text{E+04})$ & $4.63\text{E+04}\,(1.02\text{E+04})$ & $4.86\text{E+04}\,(1.01\text{E+04})$ \\
			& sumHV & $1.91\text{E+06}\,(4.49\text{E+05})$ & $2.51\text{E+06}\,(4.64\text{E+05})$ & $2.55\text{E+06}\,(5.29\text{E+05})$ & $2.25\text{E+06}\,(4.92\text{E+05})$ & $\cellcolor{lightgray} 2.60\text{E+06}\,(3.38\text{E+05})$ \\
			\midrule
			
			\multirow{3}{*}{MPUAV6} 
			& Comparative & $2.84\text{E-01}\,(1.04\text{E-01})$ & $6.01\text{E-01}\,(1.44\text{E-01})$ & $6.00\text{E-01}\,(1.35\text{E-01})$ & $5.76\text{E-01}\,(1.11\text{E-01})$ & $\cellcolor{lightgray} 6.10\text{E-01}\,(1.48\text{E-01})$ \\
			& Absolute& $2.82\text{E+04}\,(1.04\text{E+04})$ & $5.96\text{E+04}\,(1.43\text{E+04})$ & $5.95\text{E+04}\,(1.34\text{E+04})$ & $5.72\text{E+04}\,(1.10\text{E+04})$ & $\cellcolor{lightgray} 6.06\text{E+04}\,(1.47\text{E+04})$ \\
			& sumHV & $2.60\text{E+06}\,(7.37\text{E+05})$ & $3.68\text{E+06}\,(7.23\text{E+05})$ & $3.48\text{E+06}\,(6.70\text{E+05})$ & $3.31\text{E+06}\,(6.13\text{E+05})$ & $\cellcolor{lightgray} 3.83\text{E+06}\,(7.33\text{E+05})$ \\
			
			\bottomrule
	\end{tabular}}
\end{table*}

Fig.~\ref{fig:5.3.1} provides a qualitative illustration of the solution distributions on MPUAV5. In this problem, MPAIMA achieves a wider spread of solutions in the objective space, which corresponds to a larger sumHV value. By contrast, OptMPNDS2 produces solutions that are more evenly balanced between the two DMs. While this difference does not lead to a pronounced discrepancy in the PF-free quantitative metrics reported in Table~\ref{tab:UAV_path_planning_results_summary}, it highlights a qualitative distinction between efficiency-oriented coverage and fairness-oriented compromise in solution distributions.

\begin{figure}[htbp]
	\centering
	\includegraphics[width=0.85\linewidth]{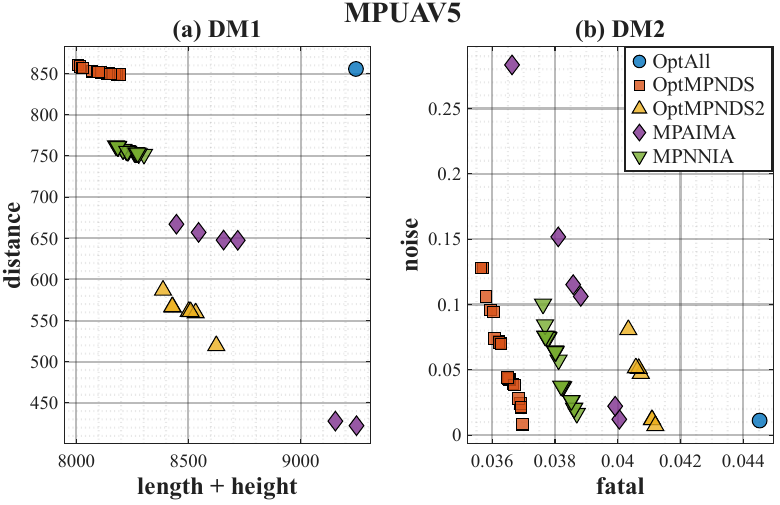}
	\caption{Solution distributions of comparative algorithms in the objective spaces of MPUAV5.}
	\label{fig:5.3.1}
\end{figure}

Overall, the UAV path planning example demonstrates that, in the absence of explicit PFs, the proposed Nash-product-based framework can still be instantiated in a PF-free manner and yields evaluation results generally aligned with conventional hypervolume-based metrics. This subsection demonstrates that the proposed Nash-product-based framework remains applicable and produces reasonable rankings in PF-free settings. How to systematically incorporate fairness penalties or concession mechanisms when PFs are unknown is an important open issue and is left for future investigation.

\section{Conclusion}
\label{sec: Conclusion}

This paper proposes a Nash-product-based performance evaluation framework for MPMOPs, with explicit emphasis on fairness across heterogeneous DMs. By transforming heterogeneous performance metrics into strictly monotone utility representations and aggregating them multiplicatively, the proposed framework favors balanced compromises over mean-dominated outcomes. To address cases where strictly common Pareto-optimal solutions do not exist, we introduce concession rates to generalize multi-party optimality and penalize solutions that deviate excessively from mutually acceptable regions. Experiments on MPMOP and MPDMP benchmarks show that the proposed metric $\Psi_{\mathrm{NP}}$ can produce rankings different from meanIGD, reflecting inter-party balance and mutual acceptability under both consensus and non-consensus settings, including symmetric and asymmetric concession configurations. The framework also remains well defined in PF-free scenarios and yields results consistent with classical indicators, although explicit fairness enforcement in such settings remains challenging. 

Future work will explore fairness-aware evaluation for PF-free problems and optimization algorithms that incorporate fairness directly into the search process.

\bibliographystyle{IEEEtran}
\bibliography{ref}

\vfill



\section*{Supplementary material}

# Supplementary document of Fairness-Aware Performance Evaluation for Multi-Party Multi-Objective Optimization

## CONTENTS

## S-I. EXPERIMENTAL PARAMETER SETTINGS

### A. Setting of the Constant $C$

To determine an appropriate value of the constant $C$, a preliminary empirical analysis is conducted using NSGA-II. Specifically, NSGA-II is independently executed for $R = 10$ runs. For each decision maker (DM) $m$, the maximum loss observed across all runs is recorded as

$$\widehat{L}_m = \max_{r=1,\ldots,R} \max_{\mathbf{x} \in \mathcal{P}_{\text{NSGA-II}}^{(r)}} L_m(\mathbf{x}),$$

where $\mathcal{P}_{\text{NSGA-II}}^{(r)}$ denotes the final population obtained in the $r$-th run of NSGA-II.

The constant $C$ is then defined as

$$C = \max_{m=1,\ldots,M} \widehat{L}_m + \xi,$$

where $\xi > 0$ is a small safety margin introduced to account for potential unseen loss values.

This setting ensures that $C$ serves as an empirical upper bound for all loss values encountered during the search process. NSGA-II is employed in this preliminary step due to its robustness in approximating diverse Pareto-optimal solutions, making it suitable for estimating conservative loss bounds.

### B. Setting of the Penalty Factor $\lambda_m$

For each DM $m$, the Pareto-optimal solutions of all other DMs are projected onto the objective subspace of DM $m$, forming the set $\mathcal{P}_{-m}$. Let $PF_m$ denote the Pareto front of DM $m$.

We first define

$$\delta_m = \max_{v \in \mathcal{P}_{-m}} \text{IGD}(v, PF_m),$$

which represents the worst-case approximation error induced by solutions optimized for other DMs when evaluated against $PF_m$.

The penalty factor $\lambda_m$ is then determined as

$$\lambda_m = \max_{v \in \mathcal{P}_{-m}} \frac{\delta_m - \text{IGD}(v, PF_m)}{\gamma_m(v) - \hat{\gamma}},$$

where $\gamma_m(v)$ denotes the concession rate of solution $v$ with respect to DM $m$, and $\hat{\gamma}$ is the concession-rate threshold specified in the evaluation function.

### C. Setting of the Concession Rate $\hat{\gamma}$

To evaluate algorithmic behavior under varying degrees of compromise among DMs, multiple experimental scenarios are further defined by specifying different concession-rate thresholds. The same set of threshold configurations is consistently applied to all benchmark categories, enabling a unified and fair comparison across problem types. The considered concession-rate settings are reported in Table S-I.

TABLE S-I: Concession Rate Threshold Configurations

| Case | Description | 2 DMs | | 3 DMs | | |
|---|---|---|---|---|---|---|
| | | $\gamma_1$ | $\gamma_2$ | $\gamma_1$ | $\gamma_2$ | $\gamma_3$ |
| 1 | No concession | 0.0 | 0.0 | 0.0 | 0.0 | 0.0 |
| 2 | Partial concession | 0.3 | 0.0 | 0.3 | 0.0 | 0.0 |
| 3 | Equal partial | 0.3 | 0.3 | 0.3 | 0.3 | 0.3 |
| 4 | Asymmetric | 0.3 | 0.7 | 0.3 | 0.1 | 0.6 |
| 5 | Equal moderate | 0.5 | 0.5 | 0.5 | 0.5 | 0.5 |
| 6 | Equal high | 0.7 | 0.7 | 0.7 | 0.7 | 0.7 |

## S-II. MPMOP12–MPMOP17 BENCHMARK PROBLEMS

This section provides detailed descriptions of the MPMOP12–MPMOP17 benchmark problems. These problems are specifically constructed such that the Pareto sets (PSs) of different DMs are mutually disjoint, i.e., no solution is Pareto-optimal for all DMs simultaneously. This property distinguishes MPMOP12–MPMOP17 from MPMOP1–MPMOP11, where common Pareto-optimal solutions across DMs may exist.

The detailed configurations of the benchmark problems are summarized in Table S-III.

## S-III. MPDMP13–MPDMP18 BENCHMARK PROBLEMS

The MPDMP13–MPDMP18 benchmark suite extends the MPDMP framework by incorporating more complex geometric configurations with multiple DMs and heterogeneous

numbers of objectives. These problems are constructed based on diverse Pareto set (PS) geometries, including line segments, triangles, squares, and pentagons, which induce varied trade-off structures among DMs. As a result, the suite provides controlled yet nontrivial test scenarios for evaluating algorithm performance in multi-DM optimization settings.

For all problems, the decision vector is defined over a two-dimensional continuous domain, i.e., $\mathbf{x} \in [0, 100]^2$. For each DM $m \in \{1, \ldots, M\}$, let

$$F_m(\mathbf{x}) = \big(f_{m,1}(\mathbf{x}), \ldots, f_{m,K_m}(\mathbf{x})\big)$$

denote its objective vector. Each objective function is defined as the Euclidean distance from the decision vector $\mathbf{x}$ to a reference point associated with the underlying geometric structure:

$$f_{m,i}(\mathbf{x}) = \sqrt{(x_1 - p_{m,i,1})^2 + (x_2 - p_{m,i,2})^2}, \quad i = 1, \ldots, K_m, \tag{1}$$

where $\mathbf{p}_{m,i} = [p_{m,i,1}, p_{m,i,2}]^\mathsf{T}$ denotes the $i$-th reference point defining the PS geometry for DM $m$, and $K_m$ is the number of objectives associated with DM $m$. The reference points for all benchmark problems are summarized in Table S-II.

TABLE S-II: Reference points for MPDMP13–MPDMP18

| Problem | DM | Objective number | PS | $p_{i,1}$ | $p_{i,2}$ |
|---|---|---|---|---|---|
| MPDMP13 | 1 | 2 | Line segment | 30<br>70 | 40<br>40 |
|  | 2 | 2 | Line segment | 30<br>70 | 60<br>60 |
| MPDMP14 | 1 | 3 | Triangle | 30<br>50<br>40 | 35<br>35<br>50 |
|  | 2 | 3 | Triangle | 50<br>70<br>60 | 50<br>50<br>65 |
| MPDMP15 | 1 | 4 | Square | 30<br>45<br>45<br>30 | 35<br>35<br>50<br>50 |
|  | 2 | 4 | Square | 55<br>70<br>70<br>55 | 50<br>50<br>65<br>65 |
| MPDMP16 | 1 | 5 | Pentagon | 30<br>40<br>50<br>50<br>35 | 40<br>30<br>35<br>45<br>50 |
|  | 2 | 5 | Pentagon | 50<br>65<br>70<br>60<br>50 | 55<br>50<br>60<br>70<br>65 |
| MPDMP17 | 1 | 2 | Line segment | 25<br>40 | 60<br>70 |
|  | 2 | 2 | Line segment | 60<br>75 | 70<br>60 |
|  | 3 | 2 | Line segment | 40<br>60 | 30<br>30 |
| MPDMP18 | 1 | 3 | Triangle | 25<br>35<br>35 | 45<br>35<br>55 |
|  | 2 | 3 | Triangle | 45<br>55<br>50 | 25<br>25<br>40 |
|  | 3 | 3 | Triangle | 65<br>75<br>75 | 45<br>35<br>55 |

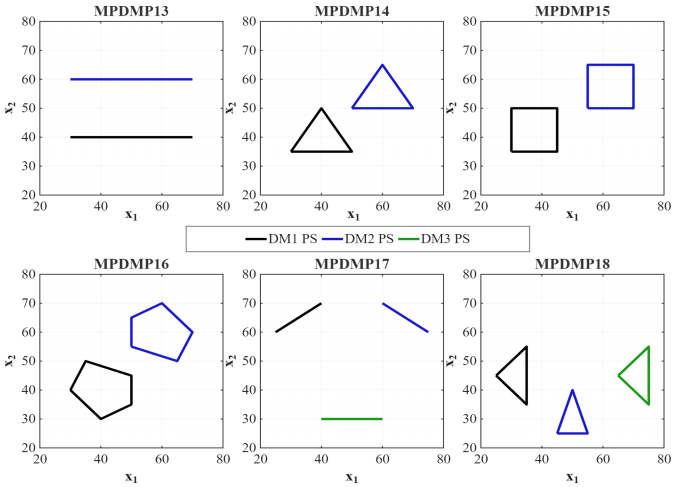

Fig. 1: MPDMPs with no common Pareto optimal solutions.

TABLE S-III: Mathematical Definitions of Benchmark Problems MPMOP12–MPMOP17

| Problem | DM | Objective Functions | Variable Bounds | Pareto Set (PS) |
|---|---|---|---|---|
| MPMOP12 | DM$_1$ | $F_1 = (f_{11}(\mathbf{x}), f_{12}(\mathbf{x}))$, $g_1(\mathbf{x}) = 1 + \sum_{i=2}^{D}(x_i - 0.5)^2$, $f_{11}(\mathbf{x}) = g_1(\mathbf{x}) \cdot \frac{2}{x_1}$, $f_{12}(\mathbf{x}) = g_1(\mathbf{x}) \cdot \frac{x_1}{2}$ | $x_1 \in [1, 2]$, $x_{i \geq 2} \in [0, 1]$ | $\{\mathbf{x} \mid x_1 \in [1, 2], x_i = 0.5\ (i \geq 2)\}$ |
|  | DM$_2$ | $F_2 = (f_{21}(\mathbf{x}), f_{22}(\mathbf{x}))$, $g_2(\mathbf{x}) = 1 + \sum_{i=2}^{D}\left(x_i - \frac{1}{1+e^{-5(x_1-3.5)}}\right)^2$, $f_{21}(\mathbf{x}) = g_2(\mathbf{x}) \cdot \frac{3}{x_1}$, $f_{22}(\mathbf{x}) = g_2(\mathbf{x}) \cdot \frac{x_1}{3}$ | $x_1 \in [3, 4]$, $x_{i \geq 2} \in [0, 1]$ | $\{\mathbf{x} \mid x_1 \in [3, 4], x_i = \frac{1}{1+e^{-5(x_1-3.5)}}\ (i \geq 2)\}$ |
| MPMOP13 | DM$_1$ | $F_1 = (f_{11}(\mathbf{x}), f_{12}(\mathbf{x}))$, $g_1(\mathbf{x}) = 1 + \sum_{i=2}^{D}\left(x_i - \frac{\sin(4\pi x_1)}{2}\right)^2 + 100\max(0, x_1 - 0.5)^2$, $f_{11}(\mathbf{x}) = g_1(\mathbf{x})(x_1 + 0.1\sin(3\pi x_1))$, $f_{12}(\mathbf{x}) = g_1(\mathbf{x})(1 - x_1 + 0.1\sin(3\pi x_1))^{4.25}$ | $x_i \in [0, 1]$ | $\{\mathbf{x} \mid x_1 \in [0, 0.5), x_i = \frac{\sin(4\pi x_1)}{2}\ (i \geq 2)\}$ |
|  | DM$_2$ | $F_2 = (f_{21}(\mathbf{x}), f_{22}(\mathbf{x}))$, $g_2(\mathbf{x}) = 1 + \sum_{i=2}^{D} x_i^2 + 100\max(0, 0.5 - x_1)^2$, $f_{21}(\mathbf{x}) = g_2(\mathbf{x})(x_1 + 0.1\sin(3\pi x_1))$, $f_{22}(\mathbf{x}) = g_2(\mathbf{x})(1 - x_1 + 0.1\sin(3\pi x_1))^{4.25}$ | $x_i \in [0, 1]$ | $\{\mathbf{x} \mid x_1 \in [0.5, 1], x_i = 0\ (i \geq 2)\}$ |
| MPMOP14 | DM$_1$ | $F_1 = (f_{11}(\mathbf{x}), f_{12}(\mathbf{x}), f_{13}(\mathbf{x}))$, $g_1(\mathbf{x}) = 1 + \sum_{i=3}^{D}\left(x_i - \frac{\sin(2\pi(x_1+x_2) + \pi/4)}{2}\right)^2$, $f_{11}(\mathbf{x}) = g_1(\mathbf{x})\sin^2\left(\frac{\pi}{2}x_1\right)$, $f_{12}(\mathbf{x}) = g_1(\mathbf{x})\sin^2\left(\frac{\pi}{2}x_2\right)\cos^2\left(\frac{\pi}{2}x_1\right)$, $f_{13}(\mathbf{x}) = g_1(\mathbf{x})\cos^2\left(\frac{\pi}{2}x_2\right)\cos^2\left(\frac{\pi}{2}x_1\right)$ | $x_i \in [0, 1]$ | $\{\mathbf{x} \mid x_1 \in [0, 0.5], x_2 \in [0, 1], x_i = \frac{\sin(2\pi(x_1+x_2)+\pi/4)}{2}\ (i \geq 3)\}$ |
|  | DM$_2$ | $F_2 = (f_{21}(\mathbf{x}), f_{22}(\mathbf{x}), f_{23}(\mathbf{x}))$, $g_2(\mathbf{x}) = 1 + \sum_{i=3}^{D}\left(x_i - \sin\left(2\pi(x_1+x_2) + \frac{\pi}{2}\right)\right)^2$, $f_{21}(\mathbf{x}) = g_2(\mathbf{x})\sin^2\left(\frac{\pi}{2}x_1\right)$, $f_{22}(\mathbf{x}) = g_2(\mathbf{x})\sin^2\left(\frac{\pi}{2}x_2\right)\cos^2\left(\frac{\pi}{2}x_1\right)$, $f_{23}(\mathbf{x}) = g_2(\mathbf{x})\cos^2\left(\frac{\pi}{2}x_2\right)\cos^2\left(\frac{\pi}{2}x_1\right)$ | $x_i \in [0, 1]$ | $\{\mathbf{x} \mid x_1 \in (0.5, 1], x_2 \in [0, 1], x_i = \sin\left(2\pi(x_1+x_2) + \frac{\pi}{2}\right)\ (i \geq 3)\}$ |

TABLE S-III: Mathematical Definitions of Benchmark Problems MPMOP12–MPMOP17 (Continued)

| Prob. | DM | Objective Functions | Variable Bounds | Pareto Set (PS) |
|---|---|---|---|---|
| MPMOP15 | $DM_1$ | $F_1 = (f_{11}(\mathbf{x}), f_{12}(\mathbf{x}))$, <br> $g_1(\mathbf{x}) = 1 + \sum_{i=2}^{D}(x_i - 0.5)^2$, <br> $f_{11}(\mathbf{x}) = g_1(\mathbf{x})\dfrac{2}{x_1}, \quad f_{12}(\mathbf{x}) = g_1(\mathbf{x})\dfrac{x_1}{2}$ | $x_1 \in [1, 2]$, <br> $x_{i \geq 2} \in [0, 1]$ | $\{\mathbf{x} \mid x_1 \in [1, 2],\ x_i = 0.5\ (i \geq 2)\}$ |
| | $DM_2$ | $F_2 = (f_{21}(\mathbf{x}), f_{22}(\mathbf{x}))$, <br> $g_2(\mathbf{x}) = 1 + \sum_{i=2}^{D}\left(x_i - \dfrac{1}{1+e^{-5(x_1-2.5)}}\right)^2$, <br> $f_{21}(\mathbf{x}) = g_2(\mathbf{x})\dfrac{3}{x_1}, \quad f_{22}(\mathbf{x}) = g_2(\mathbf{x})\dfrac{x_1}{3}$ | $x_1 \in [2, 3]$, <br> $x_{i \geq 2} \in [0, 1]$ | $\{\mathbf{x} \mid x_1 \in [2, 3],\ x_i = \dfrac{1}{1+e^{-5(x_1-2.5)}}\ (i \geq 2)\}$ |
| | $DM_3$ | $F_3 = (f_{31}(\mathbf{x}), f_{32}(\mathbf{x}))$, <br> $g_3(\mathbf{x}) = 1 + \sum_{i=2}^{D}(x_i - 0.5)^2$, <br> $f_{31}(\mathbf{x}) = g_3(\mathbf{x})\dfrac{4}{x_1}, \quad f_{32}(\mathbf{x}) = g_3(\mathbf{x})\dfrac{x_1}{4}$ | $x_1 \in [3, 4]$, <br> $x_{i \geq 2} \in [0, 1]$ | $\{\mathbf{x} \mid x_1 \in [3, 4],\ x_i = 0.5\ (i \geq 2)\}$ |
| MPMOP16 | $DM_1$ | $F_1 = (f_{11}(\mathbf{x}), f_{12}(\mathbf{x}))$, <br> $g_1(\mathbf{x}) = 1 + \sum_{i=2}^{D}\left(x_i - \dfrac{\sin(4\pi x_1)}{2}\right)^2 + 100\max(0, x_1 - 1/3)^2$, <br> $f_{11}(\mathbf{x}) = g_1(\mathbf{x})(x_1 + 0.1\sin(3\pi x_1))$, <br> $f_{12}(\mathbf{x}) = g_1(\mathbf{x})(1 - x_1 + 0.1\sin(3\pi x_1))^{2.5}$ | $x_i \in [0, 1]$ | $\{\mathbf{x} \mid x_1 \in [0, 1/3),\ x_i = \dfrac{\sin(4\pi x_1)}{2}\ (i \geq 2)\}$ |
| | $DM_2$ | $F_2 = (f_{21}(\mathbf{x}), f_{22}(\mathbf{x}))$, <br> $g_2(\mathbf{x}) = 1 + \sum_{i=2}^{D} x_i^2 + 100\left(\max(0, 1/3 - x_1)^2 + \max(0, x_1 - 2/3)^2\right)$, <br> $f_{21}(\mathbf{x}) = g_2(\mathbf{x})(x_1 + 0.1\sin(3\pi x_1))$, <br> $f_{22}(\mathbf{x}) = g_2(\mathbf{x})(1 - x_1 + 0.1\sin(3\pi x_1))^{3.0}$ | $x_i \in [0, 1]$ | $\{\mathbf{x} \mid x_1 \in [1/3, 2/3],\ x_i = 0\ (i \geq 2)\}$ |
| | $DM_3$ | $F_3 = (f_{31}(\mathbf{x}), f_{32}(\mathbf{x}))$, <br> $g_3(\mathbf{x}) = 1 + \sum_{i=2}^{D}\left(x_i + \dfrac{\sin(4\pi x_1)}{2}\right)^2 + 100\max(0, 2/3 - x_1)^2$, <br> $f_{31}(\mathbf{x}) = g_3(\mathbf{x})(x_1 + 0.1\sin(3\pi x_1))$, <br> $f_{32}(\mathbf{x}) = g_3(\mathbf{x})(1 - x_1 + 0.1\sin(3\pi x_1))^{3.5}$ | $x_i \in [0, 1]$ | $\{\mathbf{x} \mid x_1 \in [2/3, 1],\ x_i = -\dfrac{\sin(4\pi x_1)}{2}\ (i \geq 2)\}$ |

TABLE S-III: Mathematical Definitions of Benchmark Problems MPMOP12–MPMOP17 (Continued)

| Prob. | DM | Objective Functions | Variable Bounds | Pareto Set (PS) |
|---|---|---|---|---|
| **MPMOP17** | DM$_1$ | $F_1 = (f_{11}(\mathbf{x}), f_{12}(\mathbf{x}), f_{13}(\mathbf{x}))$, <br> $g_1(\mathbf{x}) = 1 + \sum_{i=3}^{D}(x_i - 0.5x_1)^2$, <br> $y_1 = \frac{\pi}{6} + \frac{\pi}{3}x_1, \quad y_2 = \frac{\pi}{2}x_2$, <br> $f_{11}(\mathbf{x}) = g_1(\mathbf{x})\sin(y_1)$, <br> $f_{12}(\mathbf{x}) = g_1(\mathbf{x})\sin(y_2)\cos(y_1)$, <br> $f_{13}(\mathbf{x}) = g_1(\mathbf{x})\cos(y_2)\cos(y_1)$ | $x_i \in [0,1]$ | $\{\mathbf{x} \mid x_1 < 0.33, x_2 \in [0,1], x_i = 0.5x_1 \ (i \geq 3)\}$ |
| | DM$_2$ | $F_2 = (f_{21}(\mathbf{x}), f_{22}(\mathbf{x}), f_{23}(\mathbf{x}))$, <br> $g_2(\mathbf{x}) = 1 + \sum_{i=3}^{D} x_i^2$, <br> $y_1 = \frac{\pi}{2}x_1, \quad y_2 = \frac{\pi}{2}x_2$, <br> $f_{21}(\mathbf{x}) = g_2(\mathbf{x})\sin(y_1)$, <br> $f_{22}(\mathbf{x}) = g_2(\mathbf{x})\sin(y_2)\cos(y_1)$, <br> $f_{23}(\mathbf{x}) = g_2(\mathbf{x})\cos(y_2)\cos(y_1)$ | $x_i \in [0,1]$ | $\{\mathbf{x} \mid 0.33 \leq x_1 < 0.66, x_2 \in [0,1], x_i = 0 \ (i \geq 3)\}$ |
| | DM$_3$ | $F_3 = (f_{31}(\mathbf{x}), f_{32}(\mathbf{x}), f_{33}(\mathbf{x}))$, <br> $g_3(\mathbf{x}) = 1 + \sum_{i=3}^{D}(x_i - 0.5x_1)^2$, <br> $y_1 = \frac{\pi}{6} + \frac{\pi}{3}x_1, \quad y_2 = \frac{\pi}{2}x_2$, <br> $f_{31}(\mathbf{x}) = g_3(\mathbf{x})\sin(y_1)$, <br> $f_{32}(\mathbf{x}) = g_3(\mathbf{x})\sin(y_2)\cos(y_1)$, <br> $f_{33}(\mathbf{x}) = g_3(\mathbf{x})\cos(y_2)\cos(y_1)$ | $x_i \in [0,1]$ | $\{\mathbf{x} \mid x_1 \geq 0.66, x_2 \in [0,1], x_i = 0.5x_1 \ (i \geq 3)\}$ |



\end{document}